\documentclass{article}

\usepackage{iclr2026_conference,times}

\usepackage{abbrevs}
\newname\PL{continuous piecewise linear} %

\usepackage{amssymb}

\usepackage[utf8]{inputenc} %

\usepackage{graphicx}

\usepackage[T1]{fontenc}

\IfFileExists{headers/config/showoverfull.config}{
	\overfullrule=1cm
}{
}

\usepackage{marginnote}

\usepackage[backgroundcolor=none,linecolor=red,textsize=footnotesize]{todonotes}

\usepackage{etoolbox}

\newbool{includeappendix}
\setbool{includeappendix}{true} %
\IfFileExists{headers/config/noappendix.config}{
	\setbool{includeappendix}{false}
}{}

\newif\ifincludeappendixx
\ifbool{includeappendix}{
	\includeappendixxtrue
}{
	\includeappendixxfalse
}

\usepackage{xr} %
\usepackage{filecontents}

\ifbool{includeappendix}{}{
	\input{appendix-labels-loader}

	\externaldocument{appendix-labels}
}

\newcommand{\eg}{e.g., }
\newcommand{\ie}{i.e., }

\definecolor{my-full-blue}{HTML}{1F77B4}

\definecolor{my-full-orange}{HTML}{FF7F0E}

\definecolor{my-full-green}{HTML}{2CA02C}

\definecolor{my-full-red}{HTML}{d62728}

\definecolor{my-full-purple}{HTML}{9467bd}

\definecolor{my-full-brown}{HTML}{8c564b}

\definecolor{my-full-pink}{HTML}{e377c2}

\definecolor{my-full-gray}{HTML}{7f7f7f}

\definecolor{my-full-olive}{HTML}{bcbd22}

\definecolor{my-full-cyan}{HTML}{17becf}

\definecolor{muted-green}{RGB}{87,132,89}
\definecolor{muted-red}{RGB}{218,90,84}

\colorlet{cgreen}{muted-green}
\colorlet{cred}{muted-red}

\colorlet{my-blue}{my-full-blue!30}
\colorlet{my-orange}{my-full-orange!30}
\colorlet{my-green}{my-full-green!30}
\colorlet{my-red}{my-full-red!30}
\colorlet{my-purple}{my-full-purple!30}
\colorlet{my-brown}{my-full-brown!30}
\colorlet{my-pink}{my-full-pink!30}
\colorlet{my-gray}{my-full-gray!30}
\colorlet{my-olive}{my-full-olive!30}
\colorlet{my-cyan}{my-full-cyan!30}

\usepackage{listings}

\usepackage{textcomp}

\usepackage{xcolor}

\usepackage[scaled=0.8]{beramono}

\definecolor{ckeyword}{HTML}{7F0055}
\definecolor{ccomment}{HTML}{3F7F5F}
\definecolor{cstring}{HTML}{2A0099}

\lstdefinestyle{numbers}{
	numbers=left,
	framexleftmargin=20pt,
	numberstyle=\tiny,
	firstnumber=auto,
	numbersep=1em,
	xleftmargin=2em
}

\lstdefinestyle{layout}{
	frame=none,
	captionpos=b,
}

\lstdefinestyle{comment-style}{
	morecomment=[l]//,
	morecomment=[s]{/*}{*/},
	commentstyle={\color{ccomment}\itshape},
}

\lstdefinestyle{string-style}{
	morestring=[b]",%
	morestring=[b]',%
	stringstyle={\color{cstring}},
	showstringspaces=false,%
}

\lstdefinestyle{keyword-style}{
	keywordstyle={\ttfamily\bfseries},
	morekeywords={
		function,
		constructor,
		int,
		bool,
		return,
		returns,
		uint
	},
	morekeywords = [2]{},
	keywordstyle = [2]{\text},
	sensitive=true,
}

\lstdefinestyle{input-encoding}{
	inputencoding=utf8,
	extendedchars=true,
	literate=
	{ℝ}{$\reals$}1%
	{→}{$\rightarrow$}1%
	{α}{$\alpha$}1%
	{β}{$\beta$}1%
	{λ}{$\lambda$}1%
	{θ}{$\theta$}1%
	{ϕ}{$\phi$}1%
}

\lstdefinestyle{escaping}{
	moredelim={**[is][\color{blue}]{\%}{\%}},
	escapechar=|,
	mathescape=true
}

\lstdefinestyle{default-style}{
	basicstyle=\fontencoding{T1}\ttfamily\footnotesize,
	style=numbers,
	style=layout,
	style=comment-style,
	style=string-style,
	style=keyword-style,
	style=input-encoding,
	style=escaping,
	tabsize=2,
	upquote=true
}

\lstdefinelanguage{BASIC}{
	language=C++,
	style=default-style
}[keywords,comments,strings]%

\usepackage{tikz}

\usetikzlibrary{arrows}
\usetikzlibrary{automata}
\usetikzlibrary{calc}
\usetikzlibrary{backgrounds}
\usetikzlibrary{decorations.markings}
\usetikzlibrary{decorations.pathmorphing}
\usetikzlibrary{decorations.pathreplacing}
\usetikzlibrary{fit}
\usetikzlibrary{patterns}
\usetikzlibrary{positioning}
\usetikzlibrary{shadows}
\usetikzlibrary{shapes}
\usetikzlibrary{shapes.geometric}

\usepackage{amsmath,amsfonts,bm}

\def\floor#1{\lfloor #1 \rfloor}
\def\1{\bm{1}}

\def\vzero{{\bm{0}}}
\def\vone{{\bm{1}}}

\def\va{{\bm{a}}}
\def\vb{{\bm{b}}}
\def\vc{{\bm{c}}}

\def\ve{{\bm{e}}}

\def\vh{{\bm{h}}}

\def\vm{{\bm{m}}}
\def\vn{{\bm{n}}}

\def\vp{{\bm{p}}}
\def\vq{{\bm{q}}}

\def\vs{{\bm{s}}}

\def\vu{{\bm{u}}}
\def\vv{{\bm{v}}}

\def\vx{{\bm{x}}}
\def\vy{{\bm{y}}}
\def\vz{{\bm{z}}}

\def\mA{{\bm{A}}}

\def\mW{{\bm{W}}}

\DeclareMathAlphabet{\mathsfit}{\encodingdefault}{\sfdefault}{m}{sl}
\SetMathAlphabet{\mathsfit}{bold}{\encodingdefault}{\sfdefault}{bx}{n}

\def\gC{{\mathcal{C}}}

\def\gM{{\mathcal{M}}}

\def\gP{{\mathcal{P}}}

\def\gS{{\mathcal{S}}}

\def\sN{{\mathbb{N}}}

\def\sR{{\mathbb{R}}}

\DeclareMathOperator*{\conv}{conv}

\usepackage[hidelinks]{hyperref}
\usepackage{url}

\usepackage{adjustbox}
\usepackage[capitalise,noabbrev]{cleveref}
\usepackage{scrextend}

\usepackage{enumitem}
\usepackage{threeparttable}
\usepackage{multirow, makecell}
\usepackage{booktabs}
\usepackage{xspace}
\usepackage{wrapfig}
\usepackage{dsfont}
\usepackage[normalem]{ulem}

\usepackage{subcaption}
\usepackage{caption}
\usepackage{dashrule}
\usepackage{amsthm}
\usepackage{thmtools}
\usepackage{thm-restate}
\usepackage{mathtools}
\usepackage[makeroom]{cancel}

\declaretheoremstyle[
  spaceabove=0.5em, %
  spacebelow=0.01em,%
  headfont=\normalfont\bfseries,
  bodyfont=\normalfont,
  postheadspace=0.5em,
]{mystyle}

\declaretheoremstyle[
  spaceabove=0.5em, %
  spacebelow=0.01em,%
  headfont=\normalfont\itshape,
  bodyfont=\normalfont,
  postheadspace=0.5em,
  qed=$\square$,
]{myproofstyle}

\declaretheorem[name=Theorem,numberwithin=section,style=mystyle]{thm}
\crefname{thm}{Theorem}{Theorems}

\declaretheorem[name=Lemma,numberlike=thm,style=mystyle]{lem}
\crefname{lem}{Lemma}{Lemmas}

\crefname{prop}{Proposition}{Propositions}

\crefname{cor}{Corollary}{Corollaries}

\declaretheorem[name=Definition,numberlike=thm,style=mystyle]{defi}
\crefname{defi}{Definition}{Definitions}
\declaretheorem[name=Remark,numberlike=thm,style=mystyle]{remark}
\crefname{remark}{Remark}{Remarks}

\usepackage{tikz}
\usetikzlibrary{calc,decorations,decorations.pathmorphing}
\usetikzlibrary{positioning,fit,arrows}
\usetikzlibrary{decorations.markings}
\usetikzlibrary{shapes,shapes.geometric}
\usetikzlibrary{shadows,patterns,snakes}
\usetikzlibrary{backgrounds,decorations.pathreplacing,automata}
\usetikzlibrary{intersections}
\usetikzlibrary{angles,quotes}
\usetikzlibrary{plotmarks}
\usetikzlibrary{patterns}
\usetikzlibrary{arrows.meta}
\usetikzlibrary{shapes.misc}
\usetikzlibrary{chains}
\usepackage{siunitx}
\newcolumntype{d}[1]{S[table-format=#1]}
\usetikzlibrary{patterns.meta}

\makeatletter
\def\extractcoord#1#2#3{
	\path let \p1=(#3) in \pgfextra{
		\pgfmathsetmacro#1{\x{1}/\pgf@xx}
		\pgfmathsetmacro#2{\y{1}/\pgf@yy}
		\xdef#1{#1} \xdef#2{#2}
	};
}
\makeatother

\usepackage{pifont}%

\newcommand{\deeppoly}{\textsc{DeepPoly}\xspace}

\newcommand{\ibp}{\textsc{IBP}\xspace}

\newcommand{\diag}{\text{diag}\xspace}

\DeclareMathOperator*{\din}{{\mathit{d}_{\text{in}}}}

\usepackage{comment}


\crefformat{section}{\S#2#1#3}

\crefrangeformat{section}{\S#3#1#4\crefrangeconjunction\S#5#2#6}

\crefmultiformat{section}{\S#2#1#3}{\crefpairconjunction\S#2#1#3}{\crefmiddleconjunction\S#2#1#3}{\creflastconjunction\S#2#1#3}

\newcommand{\crefrangeconjunction}{--}

\crefname{listing}{Lst.}{listings}
\crefname{line}{Lin.}{Lin.}
\crefname{appendix}{App.}{App.}

\newcommand{\appref}[1]{%
	\ifbool{includeappendix}{\cref{#1}}{the appendix}%
}
\newcommand{\Appref}[1]{%
	\ifbool{includeappendix}{\cref{#1}}{The appendix}%
}

\usepackage{algorithm,algpseudocode}

\title{Expressiveness of Multi-Neuron Convex \\Relaxations in Neural Network Certification}

\author{Yuhao Mao$^{\dagger*}$, Yani Zhang$^{\ddagger}$\thanks{Equal contribution} \, \& Martin Vechev$^\dagger$ \\
$^\dagger$Department of Computer Science\\
$^\ddagger$Department of Information Technology and Electrical Engineering\\
ETH Z\"urich, Switzerland\\
\texttt{\{yuhao.mao,martin.vechev\}@inf.ethz.ch, yanizhang@mins.ee.ethz.ch}
}

\iclrfinalcopy %

\begin{document}

\maketitle

\edef\savedtextfloatsep{\the\textfloatsep}
\edef\savedintextsep{\the\intextsep}
\setlength{\textfloatsep}{12pt plus 2pt minus 2pt}
\setlength{\intextsep}{10pt plus 2pt minus 2pt}
\raggedbottom

\begin{abstract}
	Neural network certification methods heavily rely on convex relaxations to provide robustness guarantees. However, these relaxations are often imprecise: even the most accurate single-neuron relaxation is incomplete for general ReLU networks, a limitation known as the \emph{single-neuron convex barrier}. While multi-neuron relaxations have been heuristically applied to address this issue, two central questions arise: (i) whether they overcome the convex barrier, and if not, (ii) whether they offer theoretical capabilities beyond those of single-neuron relaxations.
In this work, we present the first rigorous analysis of the expressiveness of multi-neuron relaxations. Perhaps surprisingly, we show that they are inherently incomplete, even when allocated sufficient resources to capture finitely many neurons and layers optimally. This result extends the single-neuron barrier to a \textit{universal convex barrier} for neural network certification. 
On the positive side, we show that completeness can be achieved by either (i) augmenting the network with a polynomial number of carefully designed ReLU neurons or (ii) partitioning the input domain into convex sub-polytopes, thereby distinguishing multi-neuron relaxations from single-neuron ones which are unable to realize the former and have worse partition complexity for the latter.
Our findings establish a foundation for multi-neuron relaxations and point to new directions for certified robustness, including training methods tailored to multi-neuron relaxations and verification methods with multi-neuron relaxations as the main subroutine.
	
\end{abstract}

\section{Introduction}
\label{sec:introduction}

Neural networks are
vulnerable to adversarial attacks \citep{SzegedyZSBEGF13}, where a small perturbation to the input can lead to misclassification. Adversarial robustness, 
which measures the robustness of a model with respect to adversarial perturbations, has received much research attention in recent years.
However, computing the exact adversarial robustness of a general neural network is  coNP-hard \citep{KatzBDJK17}, while adversarial attacks \citep{Carlini017,TramerCBM20} that try to find an adversarial perturbation can only provide a heuristic upper bound on the robustness of the model.
To tackle this issue, neural network certification has been proposed to provide robustness guarantees. 
 In the context of robustness certification, the task boils down to providing a numerical bound on the output of a neural network for all possible inputs within a given set. 
A central property of certification is \emph{completeness}, which requires the method to provide exact bounds for all cases.

 Certification methods based on 
 convex relaxations
can provide efficient certification by computing an over-approximation of the feasible output set of a given network, with certain trade-off on the precision \citep{WongK18,SinghGMPV18,WengZCSHDBD18,GehrMDTCV18,XuS0WCHKLH20}. 
They can also be incorporated in the training process to deliver models that are easy to certify \citep{ShiWZYH21,MullerE0V23,MaoM0V23,MaoBV24,PalmaBDKSL23,balauca2024overcoming}.
Due to the central role of convex relaxations 
 in the context of certified robustness, 
it is crucial to understand their theoretical properties. 

\paragraph{The Single-Neuron Convex Barrier}
Single-neuron relaxations are widely studied due to their popularity and simplicity. However, the single-neuron convex barrier \citep{SalmanY0HZ19, PalmaBBTK21} prevents single-neuron convex relaxations from providing exact bounds for general ReLU networks. 
 \citet{baader2024expressivity} further show that even the most precise single-neuron relaxation, namely 
 Triangle \citep{WongK18},
cannot exactly bound any ReLU network encoding the ``$\max$''
function in $\sR^2$. 
To overcome this limitation, multi-neuron relaxations have been proposed \citep{SinghGMPV18,MullerMSPV22,ZhangWXLLJ22}, achieving higher empirical precision. Yet, their theoretical properties remain largely unexplored.
In particular, it is unclear whether multi-neuron relaxations are able to provably bypass the convex barrier and provide complete certification for general ReLU networks, if given sufficient resources. A key challenge is that, unlike the single-neuron setting—where proving a barrier only requires exhibiting a concrete network for which the most precise single-neuron relaxation fails—a multi-neuron relaxation can always be made more precise by allocating more resources, thus this question cannot be answered via empirical studies. Moreover, solving multi-neuron relaxations is significantly more computationally expensive, making empirical exploration of their limits difficult.

\paragraph{This Work: Quantifying the Expressiveness and Completeness of Multi-Neuron Relaxations}

In this work, we formalize the notion of multi-neuron relaxations  and 
 rigorously investigate their expressiveness. We address 
 two central questions: (i) whether they overcome the single-neuron convex barrier, and if not, (ii) whether they offer fundamental advantages over single-neuron relaxations.

\paragraph{Key Contributions}
\begin{itemize}[left=10pt, topsep=0pt, parsep=0pt, partopsep=0pt]
    \item We prove that multi-neuron relaxations are inherently incomplete for general ReLU networks, even provided with sufficient resources to capture all neurons in each individual layer optimally (\cref{sec:no_cross_layer}). This incompleteness result extends to relaxations involving finitely many layers and to networks with common non-ReLU activations, such as logistic sigmoid and $\tanh$,
     establishing a universal convex barrier for neural network certification with convex relaxations (\cref{sec:cross_layer}).

    \item We prove that with equivalence-preserving network transformations, a layerwise multi-neuron relaxation can be a complete verifier, which is impossible for any single-neuron relaxation. This shows that the expressivity of general ReLU networks is preserved under multi-neuron relaxations: 
    every \PL function can be encoded by a network that is exactly bounded by some layerwise multi-neuron relaxation  (\cref{sec:equivalent_transformation}). This stands in sharp contrast to the  impossibility result established for single-neuron relaxations \citep{baader2024expressivity}: in a case study, we demonstrate that a simple network implementing the ``$\max$'' function in 
    $[0,1]^d$
    can be exactly bounded by a dimension-independent multi-neuron relaxation far weaker than required by the general theorem. 

    \item We analyze the properties of multi-neuron relaxations under convex polytope partitioning and show that their partition 
    complexity required to achieve complete certification is strictly lower than that of single-neuron relaxations (\cref{sec:input_split}).

    \item  We discuss the practical implications of the above theorems,
    including training strategies tailored to multi-neuron relaxations and verification methods with multi-neuron relaxations 
    as the main subroutine (\cref{sec:discussion}).
\end{itemize}
Aside from the prior works mentioned, an extended discussion of related work can be found in \cref{sec:related_work}.

\section{Background} \label{sec:background}

\begin{wrapfigure}{r}{.3\linewidth}
	\centering
	\vspace{-0.5cm}
	\resizebox{\linewidth}{!}{
	\begin{tikzpicture}

	\def\a{-1.8}
	\def\b{1.4}
	\def\al{0.15}
	\coordinate (a) at ({\a},{0});
	\coordinate (a1) at ({\a},{\a});
	\coordinate (b0) at ({\b},{0});
	\coordinate (b) at ({\b},{\b});
	\coordinate (c) at ({0},{0});
	\coordinate (aub) at ({\a},{\b-(\b-\a)*\al});
	\coordinate (bub) at ({\b},{(\b-\a)*\al});
	\coordinate (alb) at ({\a},{\a*\al});
	\coordinate (blb) at ({\b},{\b*\al});
	
	\node[circle, fill=black, minimum size=3pt,inner sep=0pt, outer sep=0pt] at (a) {};
	\node[circle, fill=black, minimum size=3pt,inner sep=0pt, outer sep=0pt] at (b) {};
	\node[circle, fill=black, minimum size=3pt,inner sep=0pt, outer sep=0pt] at (c) {};
	
	\fill[fill=blue!30] (a) -- (c) -- (b) -- cycle;
	\draw[black,ultra thick] (a) -- (c) -- (b);
	\draw[-] (b) -- ($(b)+(0.3,0.3)$);
	\draw[-] ({\a},-0.4) -- ({\a},0.4);
	\draw[-] ({\b},-0.4) -- ({\b},1.9);

	\node[anchor=north east,align=center,scale=0.85] at ({\a},-0.10) {$-1$};
	\node[anchor=south west,align=center,scale=0.85] at ({\b},-0.50) {$1$};	
	\node[anchor=south east,align=center,scale=0.85] at ({-0.3},0.70) {$y\leq \frac{1}{2} (x+1)$};
	\node[anchor=north west,align=center,scale=0.85] at (-{1.2},0.00) {$y\geq 0$};
	\node[anchor=north west,align=center,scale=0.85] at (0.4,0.6) {$y\geq x$};
	
	\draw[-Stealth] (-2.5, 0) -- (2.5, 0) node[right,scale=0.85] {$x$};
	\draw[-Stealth] (0, -0.4) -- (0, 1.6) node[above,scale=0.85] {$y$};

\end{tikzpicture}
	}
	\caption{Triangle relaxation of a ReLU with input $x \in [-1, 1]$.}
	\vspace{-0.5cm}
	\label{fig:ReLU_triangle}
\end{wrapfigure}

\subsection{Convex Relaxations for Certification}
Given a function $f:\sR^{d} \rightarrow  \sR^{d'}$ and a non-singleton compact domain $X \subseteq \sR^{d}$,
we denote the graph of the function $\{(x,f(x)):x\in X\}$ by $f[X]$.
The certification task boils down to computing the upper and lower bounds of $f(X):=\{f(x) | x \in X\}$,
in order to verify that these bounds meet certain requirements, \eg adversarial robustness.
To this end, convex relaxations approximate $f[X]$ by conditioned convex polytopes $S \subseteq  \sR^{d+d'}$ satisfying $S \supseteq f[X]$, where the condition depends on the concrete relaxation method.
We then take the upper and lower bounds of $S$ (projected onto $\sR^{d'}$)
as an over-approximation of the bounds of $f(X)$.
We denote by $\gC(\vx^{(1)},\ldots, \vx^{(L)})$ a set of affine constraints on the variables $\vx^{(1)},\ldots, \vx^{(L)}$. 
Its feasible set is the intersection of the feasible set of each included affine constraint. 
When context is clear, we use $\gC$ to refer to both the affine constraint set and its feasible set; 
for two constraint sets $\gC_1$ and $\gC_2$, we use $\gC_1 \land \gC_2$ to denote the combination of the constraints in $\gC_1$ and $\gC_2$, i.e., their feasible sets are intersected.
For an affine constraint set $\gC(\vx, \vy)$ dependent on $\vx$, we denote by $\pi_{\vx}(\gC)$ the projection of the feasible set
onto the $\vx$-space, which can be computed by, \eg applying the Fourier-Motzkin algorithm to remove the variables in $\gC$ other than $\vx$.
We assume the domain $X$ to be a convex polytope, \eg $L_\infty$ neighborhoods of a reference point,
which is the common practice in certification.
Such convex sets $S$ can be represented by a set of affine constraints $\gC(\vx, f(\vx))$ as well.
For example, consider the ReLU function $y = \rho(x) = \max(x, 0)$ on the domain $X=[-1, 1]$, represented by $\gC_0 = \{x \ge -1, x \le 1\}$.
One possible convex relaxation is the Triangle relaxation \citep{WongK18},
represented by the affine constraints
$\gC_1 = \{ y \ge x,   y\ge 0,  y \le  \frac{1}{2}(x+1)\}$.
\cref{fig:ReLU_triangle} illustrates this, 
where the black thick line represents $f[X]$ and the colored area stands for $S$. In this example, $\pi_{x} (\gC_0 \land \gC_1) = [-1,1]$ and $\pi_y (\gC_0 \land \gC_1) = [0,1]$.

\subsection{ReLU Network Analysis with Layerwise and Cross-Layer Convex Relaxations}\label{sec:define_layer_and_cross}

Consider a network\footnote{Unless explicitly stated otherwise, the term \emph{network} is understood as ReLU neural network.} 
$f = W_L \circ \rho \circ \dots \circ \rho \circ W_1$ where $W_j$ are the affine layers for $j \in [L]$ and $\rho$ is the ReLU function.
Denote the input variable %
by $\vx$, the first layer by $\vv^{(1)} := W_1 (\vx)$, the second layer by $\vv^{(2)} := \rho (\vv^{(1)})$, and so on\footnote{We consider affine transformation and ReLU as separate layers throughout the paper.}.
Assume the input convex polytope $X$ is defined by the affine constraint set $\gC_0(\vx)$. 
A \emph{layerwise convex relaxation} works as follows.
Given the input convex polytope\footnote{We always assume the input convex polytope is non-empty.} $\gC_0(\vx)$,
apply the convex relaxation to the first layer $\vv^{(1)} = W_1(\vx)$ to obtain a set of affine constraints $\gC_{1}(\vx, \vv^{(1)})$. 
Then, based on $ \pi_{\vv^{(1)}}(\gC_0 \land \gC_1)$, 
apply it to the second layer $\vv^{(2)} = \rho(\vv^{(1)})$ to obtain a set of affine constraints $\gC_{2}(\vv^{(1)}, \vv^{(2)})$.
Proceeding by layer, we obtain affine constraint sets $\gC_{ j+1}(\vv^{(j)}, \vv^{(j+1)})$,
for $j \in [2L-2]$.
All the constraints pertain to a single layer and no explicit constraint across layers is allowed, \eg
$\gC(\vx, \vv^{(2L-1)})$ would not appear explicitly in the above procedure.
Finally, we combine all constraints to get $\gC = \gC_0(\vx) \land \gC_{1}(\vx,\vv^{(1)})\land \cdots \land \gC_{2L-1}(\vv^{(2L-2)}, \vv^{(2L-1)})$, and solve $\gC$ to obtain the upper and lower bounds of the output variable $\vv^{(2L-1)}$. These bounds are then used to certify the network.

In contrast to layerwise relaxations which consider every layer separately, \emph{cross-layer relaxations} \citep{ZhangWXLLJ22} include constraints involving multiple consecutive layers. Concretely, let $r\in \sN^+$, 
for the network $f$ above, a cross-$r$-layer relaxation processes the first $r$ layers jointly and returns a set of affine constraints $\gC_1(\vx,\vv^{(1)}, \ldots,\vv^{(r)})$. Proceeding again by layer, we obtain affine constraint sets 
$\gC_2(\vv^{(1)},\ldots, \vv^{(1+r)}), \ldots, \gC_{2L-r}(\vv^{(2L-r-1)},\ldots,\vv^{(2L-1)})$, and the intersection of all feasible sets is solved to return bounds on $\vv^{(2L-1)}$.
We denote by $\gP_r$ the convex relaxation that always returns
the convex hull of the function graph of every $r$ adjacent layers on an input convex polytope to
the considered layers, which is, by definition, the most precise cross-$r$-layer convex relaxation, and likewise denote by $\gP_1$ the most precise layerwise (cross-$1$-layer) convex relaxation.
In other words, given a feasible set $S$ in the $\vv^{(i)}$ space, $\gP_r$ returns a constraint set equivalent to the convex hull of $\{(\vv^{(i)}, \dots, \vv^{(i+r)}) \mid \vv^{(i)} \in S\}$ for all $i$. All cross-$r$-layer relaxations cannot be made more precise than $\gP_r$ by definition.

For a set $H$, we denote its convex hull by $\conv(H)$.
For a compact set $X \subseteq \sR^d$, we denote by $\min X$ the $d$-dimensional vector
whose elements are the minimum value of points in $X$ on each coordinate. For example, $\min [0,1]^2 = [0,0]$.
Given a relaxation method $\gP$, a network $f$, and an input set $X$, we denote by $\ell(f,\gP,X)$ 
the vector of lower bounds on each dimension of $f$  computed by $\gP$ with respect to $X$; likewise we denote by $u(f,\gP,X)$ the upper bounds. In this work, we assume
linear programming is employed to solve the constraint sets generated by the convex relaxation
methods, and it always returns optimal bounds based on the constraints, without indicating the
existence or nonexistence of a feasible point attaining the bounds.
A glossary of all notations is detailed in \cref{app:notation}.

\subsection{Single-Neuron and Multi-Neuron Relaxations}\label{sec:define_single_multiple}
Within the framework of layerwise convex relaxations, the optimal constraint set on an affine layer
$\vy = \mA \vx+\vb$ is always $\gC(\vx,\vy) = \{\mA \vx+\vb - \vy \le \vzero,  -\mA \vx -\vb +\vy  \le \vzero \} $,
which translates to the equality $\vy = \mA \vx+\vb$.
Such constraints introduce no loss of precision, and thus are adopted by most convex relaxation methods. Concretely, other than \ibp, all convex relaxation methods considered in this paper use the exact constraints on affine layers.
The core difference between relaxation methods is how they handle the ReLU function.
Single-neuron relaxation methods %
process each ReLU neuron separately and disregard the interdependence between neurons,
while multi-neuron relaxations consider a group of ReLU neurons jointly.
For the vector $\vx$, let $\vx_i$ denote its $i$-th entry and $\vx_I$ be
the sub-vector of $\vx$ with entries corresponding to the indices in the set $I$.
For the ReLU layer $\vy =\rho (\vx)$ with $x \in \sR^d$,
the constraint sets
computed by single-neuron relaxations are of the form $\gC(\vx_i,\vy_i)$ with $i\in [d]$.
In contrast,
multi-neuron relaxations produce constraints of the form $\gC(\vx_{I_1}, \vy_{I_2})$ with $I_1, I_2 \subseteq [d]$. We only consider multi-neuron relaxations that are at least as precise as single-neuron relaxations, \ie for every $i \in [d]$, there exist $I_1, I_2$ such that $i \in I_1 \cap I_2$.

\citet{SinghGPV19B} propose the first multi-neuron relaxation called $k$-ReLU.
For the ReLU layer $\vy = \rho(\vx)$, it considers
at most $k$ unstable neurons jointly---we call neurons that switch their activation states within the input set as unstable, otherwise we call them stable---
and returns $\gC(\vx_I, \vy_I)$, with $I \subseteq [d], |I|\le k$.
However, $k$-ReLU reduces to the single-neuron Triangle relaxation in some cases (see the case study in \cref{sec:equivalent_transformation}), thus we consider a stronger multi-neuron relaxation which only restricts the number of output variables in the constraints, 
allowing $\gC(\vx, \vy)$ to be of the form $\gC(\vx, \vy_I)$ with $I \subseteq [d], |I|\le k$. Similar tricks are also used in \citet{TjandraatmadjaA20}.
We denote this special multi-neuron relaxation as $\gM_k$,
and assume it always computes the convex hull of $(\vx, \rho(\vx_I))$,
while only one index set $I$ is allowed per ReLU layer.
We emphasize that $\gM_k$ is allowed to consider unstable and stable neurons together, 
while $k$-ReLU only considers unstable neurons and the corresponding inputs jointly, thus $\gM_k$ is more precise even when $k$-ReLU also computes the convex hull of the considered variables.
Neurons that are not considered by a multi-neuron relaxation are processed by the single-neuron Triangle relaxation.
For ReLU networks of width no more than $k$, $\gM_k$, as a layerwise relaxation, is equivalent to the most precise layerwise relaxation $\gP_1$. We note that $\gP_r$  is a multi-neuron relaxation by definition, for every $r \in \sN^+$.
A toy example is provided in \cref{app:illustration} to further illustrate the concepts introduced above. We refer interested readers to \citet{baader2024expressivity} for a more detailed introduction to concrete single-neuron and multi-neuron relaxation methods.

\section{Layerwise Multi-Neuron Incompleteness}\label{sec:no_cross_layer}

In this section, we establish the incompleteness result for layerwise multi-neuron relaxations. 
We consider $\gP_1$, the most precise layerwise multi-neuron relaxation by definition, and show that it is incomplete, and the relaxation error can be arbitrarily large. 
This result naturally extends to all layerwise ReLU network verifiers, as 
they cannot be more precise than $\gP_1$.

We start with a simple example to demonstrate the idea. 
Consider the input set $X=[-1,1]^2$ and the ReLU network $f = f'\circ \rho\circ W_1$, 
where $f' = \rho(\vx_1 -1)+\rho(1-\vx_1)+\rho(\vx_2-1)+\rho(1-\vx_2)$ encodes the function $f'(\vx_1,\vx_2) = |\vx_1-1|+|\vx_2-1|, \vx\in\sR^2$, and $W_1$ is the affine transformation
$  W_1(\vx) := \begin{pmatrix}
    -1 &-1.5 \\
   -1 &1.5
\end{pmatrix} \vx+ \begin{pmatrix}
    -0.5 \\-0.5
\end{pmatrix}, \text{ for } \vx \in \sR^2.$
Let $\vu:=\rho(W_1(\vx))$. 
As illustrated in \cref{fig:example}, the affine layer $W_1$ and the subsequent ReLU transform the input set into the polytope union $U =\{\vu_1\geq 0, \vu_2\geq 0, \vu_1+\vu_2\leq 1\} \cup \{1\leq \vu_1\leq 2, \vu_2=0\} \cup\{1\leq \vu_2\leq 2, \vu_1=0\}$. 
The minimal value of $f$ on $X$ is thus $\min f(X)  = \min f'(U)  = 1$.
However, we will show $\ell(f,\gP_1,X) \leq 0$,
hence it is impossible to obtain the exact lower bound. To see this, consider the specific point $\vu^* = (1,1)$. 
On one hand, since $\gP_1$ is a sound convex relaxation, the affine constraints obtained on the layer $\rho$ and $W_1$ characterize a convex superset of $U$, thus a superset of the convex hull of $U$ which contains $\vu^*$.  
On the other hand, since $\gP_1$ prohibits affine constraints across nonadjacent layers, 
the affine constraints induced by the subsequent layers $f'$ cannot remove $\vu^*$ from the feasible set (formalized later in \cref{lem:do_not_reduce_feasible_set}). Hence, the returned lower bound satisfies $\ell(f,\gP_1,X) \leq f'(\vu^*)=0$.

We observe a general phenomenon from the example above: for a ReLU network $f=f_2 \circ f_1$, 
where $f_1$ and $f_2$ are its subnetworks,
if (1) $f_1$ maps the input set to a set $U$ whose convex hull is its strict 
superset, that is, $ \conv (U) \setminus U \ne \emptyset$, 
and (2) the subsequent network $f_2$ attains its extremal values at some point $u\in  \conv (U) \setminus U$, then a layerwise convex relaxation method \emph{cannot provide} exact bounds on $f$ for the given input set. 
This reveals a fundamental limit of layerwise multi-neuron verifiers:
there exist networks for which no verifier can provide exact bounds. 
In other words, all layerwise multi-neuron relaxations are incomplete, regardless of how many neurons in a single layer are jointly considered.
Further, as we shall show next, the relaxation error can be unbounded. 
The rest of this section is devoted to formalizing and proving the ideas above.

We first establish two lemmata characterizing properties of layerwise convex relaxations. \cref{lem:do_not_reduce_feasible_set} below states that affine constraints induced by layerwise convex relaxations 
on some hidden layer
cannot reduce the feasible set on its preceding layers.

\begin{restatable}{lem}{growfeasibleset} \label{lem:do_not_reduce_feasible_set}
        Let $L\in \sN$ and let $X$ be a convex polytope. 
        Consider a ReLU network $f =f_L\circ \cdots \circ f_1$.
        Denote the variable of the $j$-th hidden layer of $f$ by $\vv^{(j)}$, for $j\in [L-1]$, and the variable of the output layer by $\vv^{(L)}$.
    For $1\leq i<L$,
        let $\gC_1(\vx, \vv^{(1)},\ldots,\vv^{(i)} )$ and $\gC_2(\vx, \vv^{(1)},\ldots,\vv^{(L)} )$ be the set of all constraints 
    obtained by applying $\gP_1$ to the first $i$ and $L$ layers of $f$, respectively. Then, $ \pi_{\vv^{(i)}}(\gC_1(\vx, \vv^{(1)},\ldots,\vv^{(i)} )) = \pi_{\vv^{(i)}}(\gC_2(\vx, \vv^{(1)},\ldots,\vv^{(L)} )).$
    \end{restatable}

The proof is based on the definition of layerwise convex relaxations and is straightforward; we defer it to \cref{app:proof_lem:do_not_reduce_feasible_set}.
\cref{lem:do_not_reduce_feasible_set} shows that 
the constraints induced by the deeper-than-$i$ layers do not affect the feasible set of $\vv^{(i)}$. 
Despite the simplicity, this observation leads to \cref{lem:bound_cannot_improve}, 
which states that the bounds computed by $\gP_1$ cannot be better than splitting the network into two subnetworks 
at some hidden layer
and then computing their convex hulls separately.

\begin{restatable}{lem}{boundcannotimprove}\label{lem:bound_cannot_improve}
    Let $X$ be a convex polytope and 
    consider a network $f := f_2 \circ f_1$, where $f_1$ and $f_2$ are its subnetworks.
    Then, $ \ell(f, \gP_1, X) \le \min (f_2 (\conv(f_1(X))))$ and $ u(f, \gP_1, X) \ge \max (f_2 (\conv(f_1(X))))$.
    \end{restatable}

The proof of \cref{lem:bound_cannot_improve} is as follows: 
for $f_1$, 
the best approximation that a convex relaxation can attain is the convex hull of the output set of $f_1$; 
as a consequence of 
\cref{lem:do_not_reduce_feasible_set}, when processing $f_2$, 
$\gP_1$ will take the whole set $\conv(f_1(X))$ into account.
Thus, the best bound that $\gP_1$ can achieve is no better than 
bounding $f_2(\conv(f_1(X)))$.
The detailed proof of \cref{lem:bound_cannot_improve} is deferred to \cref{app:proof_lem:bound_cannot_improve}.

Now we are ready to show that the layerwise multi-neuron relaxation $\gP_1$ is incomplete.
\begin{restatable}{thm}{infiniterelaxationerror} \label{thm:infinite_relaxation_error}
        Let $d \in \sN$ and let $X$ be a convex polytope in $\sR^d$. 
        For every $0<T<\infty$,
        there exists a ReLU network $f:\sR^d\rightarrow \sR$ such that $\ell(f,\gP_1, X) \le \min f(X) - T$, and a ReLU network $g:\sR^d\rightarrow \sR$ such that
        $ u(g,\gP_1, X) \ge \max g(X) + T$.
        \end{restatable}

The proof is deferred to \cref{app:proof_thm:infinite_relaxation_error}. Informally, we construct a network $f$ such that the convex hull of the output set of the first subnetwork is a strict superset of the output set, and the subsequent layers attain its extreme values at points outside the reachable set. The construction is similar to the example provided at the beginning of this section. 
Then, we can scale the weights of the output layer by a large enough constant to make the relaxation error arbitrarily large.

\cref{thm:infinite_relaxation_error} is an unfortunate result for layerwise multi-neuron relaxations. It shows that every layerwise convex relaxation has a failure case where the relaxation error is arbitrarily large, though calculating them, e.g., $\gP_1$, is already computationally expensive for large networks. 

\begin{figure}
	\centering
    \vspace{-1em}
	\resizebox{.7\linewidth}{!}{
	\begin{tikzpicture}
    \def\scalemath{0.45}
    \def\opacity{0.6}

	\def\a{0.3}
  
	\def\b{1.4}
	\def\al{0.15}
	\coordinate (a) at (-\a, \a);
    \coordinate (b) at (\a,\a);
    \coordinate (c) at (\a,-\a);
    \coordinate (d) at (-\a,-\a);

	\fill[fill=blue!30] (a) -- (b) -- (c) -- (d)-- cycle;
	
    \node at (1.5,0.0) {$\Rightarrow$};
    \node[scale =\scalemath] at (1.5,0.25) {$\vy = W_1 (\vx)$};

    \coordinate (a1) at ({2.1},{0});
    \coordinate (b1) at ({2.7},{0.6});
    \coordinate (c1) at  ({3.6},{-0.3});
    \coordinate (d1) at ({3},{-0.9});

	\fill[fill=blue!30] (a1) -- (b1) -- (c1) -- (d1)-- cycle;
    \begin{scope}
        \clip (a1) -- (b1) -- (c1) -- (d1)-- cycle;
        \fill[pattern=grid, pattern color=black, opacity=0.3] (a1) -- (b1) -- (c1) -- (d1);
    \end{scope}

    \node at (4.7,0.0) {$\Rightarrow$};
    \node[scale =\scalemath] at (4.7,0.25) {$\vu = \rho(\vy)$};

    \coordinate (a2) at (6,0);
    \coordinate (b2) at (6.3,0);
    \coordinate (c2) at (6, 0.3);
    \coordinate (d2) at  (6, 0.6);
    \coordinate (e2) at (6.6, 0);
    \coordinate (u) at (6.3, 0.3);
    \fill[black] (u) circle (0.8pt);
    \node[above right,scale=\scalemath] at (u) {$u^*$};
    \draw[very thick,blue!30]  (a2) -- (d2);
    \draw[very thick, blue!30]  (a2) -- (e2);
    \fill[fill=blue!30] (a2) -- (b2) -- (c2) -- cycle;

            \begin{scope}
                \clip (a2) -- (d2) -- (e2) -- cycle;
                \fill[pattern=grid, pattern color=black, opacity=0.3] (a2) to (d2)
                to (e2);
            \end{scope}

    \draw[-Stealth] (-0.8, 0) -- (0.8, 0) node[right,scale=\scalemath] {$x_1$};
    \draw[-Stealth] (0, -0.8) -- (0,0.8) node[above,scale=\scalemath] {$x_2$};
    \draw[-Stealth] (3-0.8, 0) -- (3+0.8, 0) node[right,scale=\scalemath] {$y_1$};
	\draw[-Stealth] (3, -0.8) -- (3, 0.8) node[above,scale=\scalemath] {$y_2$};
    \draw[-Stealth] (6-0.8, 0) -- (6+0.8, 0) node[right,scale=\scalemath] {$u_1$};
	\draw[-Stealth] (6, -0.8) -- (6, 0.8) node[above,scale=\scalemath] {$u_2$};
\end{tikzpicture}
	}
	\caption{Blue area shows how the input box transforms under $W_1$ and ReLU; shaded area is the feasible set computed by $\gP_1$.}
	\label{fig:example}
    \vspace{-1em}
\end{figure}

\section{Cross-Layer Multi-Neuron Incompleteness}\label{sec:cross_layer}
For networks of $L$ layers, $\gP_L$ can provide exact bounds as it computes the convex hull of the input-output function. 
Since $\gP_1$ is proven incomplete in \cref{sec:no_cross_layer}, the natural question is whether there exists some $r\in \sN^+$ for $\gP_r$
to be complete.
Instead of fixing $r$ to be a constant, we consider this question in its full generality by allowing $r$ to depend on $L$ and ask:
does there exist $\alpha \in (0,1)$ such that $\gP_{\max(1, \lfloor \alpha L\rfloor)}$ provides exact bounds for all networks with $L$ layers? 
Our result is rather surprising: no such $\alpha$ exists. 
This directly implies the incompleteness of $\gP_r$ for all $r\in \sN^+$. 
Thus, the commonly believed ``single-neuron'' barrier of convex relaxations is actually a misnomer, as it extends to every multi-neuron convex relaxation, and should be renamed \emph{the universal convex barrier}.

The key insight behind our result is that for every fixed $\alpha \in (0,1)$, the cross-layer relaxation $\gP_{\max(1, \lfloor \alpha L\rfloor)}$ shares similar limitations to $\gP_1$ for certain networks. 
Formally,
\begin{restatable}{lem}{pumpinglemma} \label{lem:pumping_lemma}
  Let $\alpha \in (0,1), d,d',L_1,L_2 \in \sN^+$, and $X\subseteq \sR^d$ be a convex polytope.
  For every $L_1$-layer network $f_1: \sR^{d} \rightarrow \sR^{d'}$
 and $L_2$-layer network $f_2: \sR^{d'} \rightarrow \sR$, there exist $L> L_1+L_2$ and a $L$-layer network $f$ such that (i) $f(\vx)=f_2 \circ f_1(\vx)$, for $\forall \vx \in X$, and (ii)
 $ \ell(f, \gP_{\max(1, \lfloor \alpha L\rfloor)}, X) \le \min f_2(\conv(f_1(X)))$ and $ u(f, \gP_{\max(1, \lfloor \alpha L\rfloor)}, X) \ge \max f_2(\conv(f_1(X))).$
\end{restatable}

\cref{lem:pumping_lemma} extends \cref{lem:bound_cannot_improve} to cross-layer convex relaxations.
The idea behind its proof is similar to the pumping lemma: the original network $f_2 \circ f_1$ is pumped by adding dummy identity layers between $f_1$ and $f_2$. While cross-layer relaxations allow direct information exchange across  layers to improve bound preciseness, the pumped dummy layers block this information exchange, thereby disabling the relaxation from providing exact bounds. The formal proof is deferred to \cref{app:proof_lem:pumping_lemma}. We note that, however, only direct information exchange between $f_1$ and $f_2$ is blocked by this construction, and the cross-layer relaxation is free to provide exact bounds for both $f_1$ and $f_2$, which is easily done by $\gP_{\max(1, \lfloor \alpha L\rfloor)}$ when $\alpha \rightarrow 1$ for large enough $L$. This is also the key difference between layerwise and cross-layer relaxations. Nevertheless, merely blocking this information is sufficient to make the relaxation incomplete, as shown in \cref{thm:cross_layer}.

\begin{restatable}{thm}{crosslayerincompleteness}\label{thm:cross_layer}
Let $d\in \sN$ and let $X\subset \sR^d$ be a convex polytope.
For every $\alpha \in (0,1)$ and every constant $T>0$, there exists a network $f:\sR^d\rightarrow \sR$ such that 
$  \ell(f,\gP_{\max(1, \lfloor \alpha L\rfloor)}, X) \le \min f(X) - T$,   and a network $g:\sR^d\rightarrow \sR$ such that
 $u(g,\gP_{\max(1, \lfloor \alpha L\rfloor)}, X) \ge \max g(X) + T$.
\end{restatable}

The proof is based on the construction when proving \cref{thm:infinite_relaxation_error}. Specifically, we take the construction therein and apply \cref{lem:pumping_lemma} to obtain a deeper network that has the same semantics. Then, since the convex hull and the exact output set of $f_1$ do not completely overlap, we use a similar argument as in the proof of \cref{thm:infinite_relaxation_error} to show that the $\gP_{\max(1, \lfloor \alpha L\rfloor)}$ relaxation is incomplete for every $\alpha \in (0,1)$. The formal proof is deferred to \cref{app:proof_thm:cross_layer}. This result directly extends to $\gP_{\max(k, \lfloor \alpha L\rfloor)}$ for every constant $k \in \sN^+$.

The implication of \cref{thm:cross_layer} is daunting: even though $\gP_{\max(1, \lfloor \alpha L\rfloor)}$ is much more powerful than every practical convex relaxation algorithm, it is still incomplete, and the bounding error can be arbitrarily large. This shows a hard threshold in the completeness of cross-layer convex relaxation verifiers: $\gP_{\lfloor \alpha L \rfloor}$ is complete when $\alpha=1$ and incomplete when $\alpha<1$.

\textbf{Beyond the ReLU activation.} The cross-layer incompleteness result also generalizes to common non-ReLU activations, such as logistic sigmoid and $\tanh$ (\cref{app:non-relu}). For ReLU networks, the relative bounding error can also be arbitrarily large (\cref{app:relative_error}).

\section{Making Multi-Neuron Verifiers Complete} \label{sec:augmented_multi_neuron}

We have shown in \cref{sec:no_cross_layer} and \cref{sec:cross_layer} that no multi-neuron relaxation can be complete. 
In this section, we study techniques to augment 
  multi-neuron methods into complete verifiers. 
First, \cref{sec:equivalent_transformation} shows 
that a layerwise multi-neuron relaxation, specifically $\gP_1$, can be turned into a complete verifier by an equivalence-preserving structural transformation. 
While this result does not directly yield a practical algorithm, 
an immediate corollary is that every \PL function can be expressed by a ReLU network that is 
exactly bounded 
by a layerwise multi-neuron relaxation, which is unattainable by single-neuron relaxations. 
Second, \cref{sec:input_split} shows a sufficient and necessary condition for $\gP_1$ to be complete under a convex polytope partition, and that single-neuron relaxations inherently require more partitions to be complete.

\subsection{Completeness via Network Transformations}\label{sec:equivalent_transformation}

In this section, we consider a strong layerwise multi-neuron relaxation, namely $\gP_1$, and show that it can be turned into a complete verifier by equivalence-preserving structural transformation of the network.
Given a network $f$ to be verified, we can always construct a network $g$ equivalent to $f$ but structurally more amenable to $\gP_1$,
so as to enable exact bounds.
We formally state it in \cref{thm:precise_bounds}.

\begin{restatable}{thm}{precisebounds}\label{thm:precise_bounds}
For $d,d'\in \sN^+$, let $f:\sR^{d}\rightarrow \sR^{d'}$ be a network and 
let $X\subseteq \sR^{d}$ be a convex polytope. 
There exists a network $g:\sR^{d}\rightarrow \sR^{d'}$ satisfying $g=f$ on $X$, such that
$
\ell(g, \gP_1, X) = \min f(X)
$
and
$
u(g, \gP_1, X) = \max f(X).
$
\end{restatable}
The high-level idea is as follows: although $\gP_1$ may lose precision inside the original network, we can append an equivalence-preserving clipping subnetwork using the exact coordinatewise extrema of $f$ on $X$. This clipping subnetwork is the identity on $f(X)$, while its layerwise convex relaxation certifies that every relaxed output remains inside the exact coordinatewise range of $f(X)$. Detailed proof is deferred to \cref{app:proof_thm:precise_bounds}.

\cref{thm:precise_bounds} shows that $\gP_1$ is powerful enough for complete certification if an equivalence-preserving transformation is allowed. 
While the transform used in the proof is not practical as it directly uses the exact bounds of $f$ on $X$, 
the core message from \cref{thm:precise_bounds} is that the expressivity of ReLU networks is no longer limited by the relaxation.
As mentioned in \cref{sec:introduction},
under single-neuron relaxations, %
the expressivity of exactly bounded ReLU networks is limited to 1-D continuous piecewise linear functions \citep{baader2024expressivity}: beyond 1-D, even 
the simple ``$\max$'' function in $[0,1]^2$ cannot be encoded by a ReLU network that is exactly bounded by single-neuron relaxations. In contrast, an immediate corollary of \cref{thm:precise_bounds} is that multi-neuron preserves the full expressivity of ReLU networks as representers of general \PL functions:
\begin{restatable}{cor}{multineuronexpressiveness} \label{prop:multi_neuron_expressiveness}
  For $d\in \sN^+$, let $f:\sR^{d}\rightarrow \sR$ be a continuous piecewise linear function, and
  let $X\subseteq \sR^{d}$ be a convex polytope. 
  There exists a network $g:\sR^{d}\rightarrow \sR$ satisfying $g=f$ on $X$, such that
  $
  \ell(g, \gP_1, X) = \min f(X)
  $
  and
  $
  u(g, \gP_1, X) = \max f(X).
  $
\end{restatable}

\cref{prop:multi_neuron_expressiveness} shows that for every continuous piecewise linear function, 
there exists a ReLU network encoding it that can be exactly bounded by $\gP_1$. 
In practice, a multi-neuron relaxation much weaker  than $\gP_1$ may be enough for exact bounds. 
To illustrate this, 
we examine the concrete example of the ``$\max$'' function in $[0,1]^d$ and 
show that $\gM_1$ is sufficient to exactly bound a network encoding it in the following.

\textbf{Case study: $\max(x_1, x_2, \dots, x_d)$ can be exactly bounded by $\gM_1$ in $[0,1]^d$.} 
\begin{wrapfigure}{r}{0.4\linewidth}
	\centering
	\vspace{-0.2cm}
	\resizebox{\linewidth}{!}{
	\begin{tikzpicture}
    \tikzset{comp_node/.style={circle,draw=black,fill=white,inner sep=0pt,minimum size=13pt}}
    \def\xsep{2}
    \def\ysep{1}
    \begin{scope}
        \node[comp_node] (x1) at (0,0) {$x_1$};
        \node[comp_node] (x2) at (0,-\ysep) {$x_2$};
        \node[comp_node] (a) at ($(x1) + (\xsep,0)$) {$a$};
        \draw[-Stealth] (x1) -- (a) node[midway,above] {$1$};
        \draw[-Stealth] (x2) -- (a) node[midway,above] {$-1$};

        \node[comp_node] (b) at  ($(a) + (0,-\ysep)$) {$b$};
        \draw[-Stealth] (x2) -- (b) node[midway,below] {$1$};

        \node[comp_node] (c) at ($(a) + (\xsep,0)$) {$c$};
        \node[comp_node] (d) at ($(b) + (\xsep,0)$) {$d$};
        \draw[-Stealth] (a) -- (c) node[midway,above] {$\rho$};
        \draw[-Stealth] (b) -- (d) node[midway,below] {$\rho$};

        \node[comp_node] (f) at ($(c) + (\xsep,-.5\ysep)$) {$f$};
        \draw[-Stealth] (c) -- (f) node[midway,above] {$1$};
        \draw[-Stealth] (d) -- (f) node[midway,below] {$1$};
    \end{scope}
\end{tikzpicture}
	}
	\caption{A network encoding $f=\max(x_1, x_2)$.}
	\vspace{-0.4cm}
	\label{fig:max_case_study}
\end{wrapfigure}

First, consider the case $d=2$. The function range is $[0,1]$.
 We can represent the ``$\max$'' function by the ReLU network $f=x_2 + \rho(x_1-x_2)$, as illustrated in \cref{fig:max_case_study}. 
This network has width two (nodes $c$ and $d$) and one unstable neuron (node $c$).
Recall that $\gM_1$ computes the convex hull of $(\vx, \rho(\vx_i))$ for some $i$ for each ReLU layer.

We now show that $\gM_1$ computes the exact bounds of $f$. The input box is defined by the constraints $\{ x_1 \ge 0, x_1 \le 1, x_2 \ge 0, x_2 \le 1\}$. Besides, the constraints on the affine layers are $\{a = x_1 - x_2, b=x_2, f = c+d\}$. 
Under these constraints, we compute bounds of the neurons of the first affine layer by linear programming,
yielding $a \in [-1, 1]$ and $b \in [0,1]$. 
For the stable node $d$, the constraint is then $\{d = b\}$. 
For the unstable node $c$, the constraint is $\{c\ge 0, c \ge a, c \le 1-b, c \le a+b\}$, where the first two inequalities are based on the property of the ReLU function and the last two inequalities are based on the capability of $\gM_1$ to compute $\conv((a,b,c))$ given $\conv((a,b))=\{a \ge -b, a \le 1-b, b \in [0,1], a\in [-1, 1]\}$. Note that $\conv((a,b))$ is provided to $\gM_1$ because only a single affine layer is parsed before the ReLU layer. Therefore, we have $f = c+d = c + x_2 \ge 0 + x_2 \ge 0$ and $f = c+d = c + x_2 \le 1-b + x_2 \leq 1$.
Thus, $\gM_1$ returns the exact upper and lower bounds.
We remark that $k$-ReLU, equivalent to the Triangle relaxation in this case for every $k \ge 1$ since there is only one unstable neuron, induces on node $c$ the constraint set $\{c\ge 0, c \ge a, c \le 0.5a +0.5\}$. The resulting upper bound is $1.5$, which is inexact, consistent to \citet{baader2024expressivity}.

Based on the 2-D case, we extend the result to $[0,1]^d$. Indeed, we can rewrite ``$\max$'' in a nested form according to $\max(x_1, x_2, \dots, x_d) = \max(\max(x_1, x_2), \dots, x_d)$. By the previous argument, a multi-neuron relaxation can bound $u = \max(x_1, x_2)$ exactly. Note that $u$ has no interdependency with $x_3, \dots, x_d$, 
thus we can repeat the procedure above for $\max(u, x_3, \dots, x_d)$. By induction on $d$, a multi-neuron relaxation, namely $\gM_1$, can bound the output of a ReLU network expressing the ``$\max$'' function in $[0,1]^d$ exactly.

\subsection{Completeness via Convex Polytope Partitioning}\label{sec:input_split}

In this section, we discuss how to achieve completeness for general networks (without transformation)
by partitioning the input set into convex sub-polytopes.

Branch-and-bound  (BaB) is currently the most effective complete verifier. 
It progressively divides the current problem into subproblems, solves each subproblem recursively, and combines the results to yield the bounds. 
With a similar strategy---we call it polytope partitioning---$\gP_1$ can be turned into a complete verifier.
The idea is to partition the input set of every layer into smaller convex polytopes so that $\gP_1$ exactly bounds each of them. 
The exact bounds of the original input set can then be obtained by aggregating
bounds of the smaller polytopes. 
An algorithm is provided in \cref{app:polytope_partition}.

We first prove completeness, \ie polytope partitioning enables $\gP_1$ to calculate exact bounds.

\begin{restatable}{prop}{convexpolytopeprop}\label{prop:convex_polytope_prop}
  Let $L\in \sN$ and $d_0,d_1,\ldots,d_{L+1} \in \sN^+$. Consider an input set $X\subset \sR^{d_0}$ and a network $f
  = W_{L+1}\circ \rho \circ \cdots \circ \rho \circ W_1$,
  where $W_j:\sR^{d_{j-1}}\rightarrow \sR^{d_j}$  are the associated affine transformations for $j \in [L+1]$.
Denote the subnetworks of $f$ by $f_j:=W_{j+1}\circ \rho \circ \cdots \circ \rho \circ W_1$, for $j \in [L]$.
Assume %
$H_1, \ldots, H_\nu \subseteq X $ such that $H_1,\ldots, H_\nu$  are convex polytopes,
$f(X) = f(H_1)\cup \cdots \cup f(H_\nu)$,
and $f_j(H_k)$ is a convex polytope for all $j\in[L]$ and $k\in [\nu]$, then 
\[
\min f(X) = \min_{k\in [\nu]} \ell(f, \gP_1, H_k) \quad \quad \max f(X) = \max_{k\in [\nu]} u(f, \gP_1, H_k)
\]
\end{restatable}

\cref{prop:convex_polytope_prop} states that when we partition the input set into a finite collection of convex polytopes, 
such that each polytope remains as a convex polytope through the subsequent layers, then $\gP_1$ can 
return exact bounds on the input set. 
The proof of \cref{prop:convex_polytope_prop} (c.f. \cref{proof_prop:single_polytope_preciseness})  is based on investigating how affine and ReLU layers transform polytopes. Essentially, an affine transformation converts 
an input convex polytope into a convex polytope in the output space, and the ReLU function transforms a convex polytope into a union of convex polytopes. 
See \cref{fig:polytope_preserve} for a visualization.
We note that the conditions in 
\cref{prop:convex_polytope_prop} 
are not only 
sufficient, but also necessary: 
 if there is a sub-polytope that is no longer a 
 convex polytope after some layer, then the convex hull of the output set of that layer 
 on this sub-polytope
 is strictly larger than the actual feasible set. 
 From the discussion in \cref{sec:no_cross_layer}, we have already known $\gP_1$ cannot return exact bounds for general networks when this occurs.

How many subproblems are needed in the worst case to obtain exact bounds, and how does this number compare between single-neuron relaxations and $\gP_1$?
To make this comparison precise, we define the BaB procedure and its partition complexity.

\begin{defi}[Complete branch-and-bound with convex relaxation]
  Let $\gP$ be a convex relaxation, $X$ a convex input polytope, and $f$ a ReLU network.
  Branch-and-bound with $\gP$, denoted by $\text{BaB}(\gP)$, starts with the original problem of bounding $f$ on $X$.
  At each iteration, it selects an unresolved subproblem and computes its bounds using $\gP$.
  If these bounds are \textbf{guaranteed to be exact}, the subproblem is closed.
  Otherwise, it selects an unstable ReLU neuron with preactivation $z$ and splits the subproblem into two children corresponding to $z \leq 0$ and $z \ge 0$.
  The algorithm terminates when all terminal subproblems are closed.
  It returns the coordinatewise minimum of their lower bounds and maximum of their upper bounds.
\end{defi}

\begin{defi}[Partition complexity]\label{def:partition_complexity}
  With the notation above, the partition complexity of $\text{BaB}(\gP)$ on $f$ over $X$, denoted by $\text{\#Partition}(\gP,f,X)$, is the maximum number of closed terminal subproblems at termination over all admissible branching orders.
  Distinct terminal nodes are counted separately, even if their input regions overlap.
\end{defi}

The choice of convex relaxation determines the stopping rule used to guarantee exact bounds.
As specified in \cref{sec:define_layer_and_cross}, the relaxation returns bounds without indicating whether they are attained by the original network.
The single-neuron BaB procedure considered here closes a subproblem only when every neuron is stable, which guarantees exactness because the remaining network constraints are affine.
The example in \cref{app:example_partition} shows that resolving all unstable neurons can be necessary even for the triangle relaxation, the most precise single-neuron relaxation.
In contrast, $\gP_1$ closes a subproblem as soon as its input region and all its prefix images are convex polytopes, as justified by \cref{proof_prop:single_polytope_preciseness}.
This condition holds whenever all neurons are stable and may also hold earlier.
These stopping rules yield the following worst-case comparison.

\begin{restatable}{prop}{partition}\label{prop:partition}
  Let $\text{BaB}(\gS)$ and $\text{BaB}(\gP_1)$ be complete BaB procedures using a single-neuron relaxation $\gS$ and the multi-neuron relaxation $\gP_1$, respectively, with the stopping rules described above.
  For every ReLU network $f$ and every nonempty convex input polytope $X$,
  \[
    \text{\#Partition}(\gP_1,f,X) \leq \text{\#Partition}(\gS,f,X).
  \]
\end{restatable}

The proof is deferred to \cref{proof_prop:single_polytope_preciseness}.
The bound is tight: when $f$ has no unstable neurons, both partition complexities equal $1$.
The inequality can be strict because $\gP_1$ may stop while some neurons remain unstable.
The example in \cref{app:example_partition} shows that polytope partitioning with $\gP_1$ can have \emph{exponentially smaller partition complexity and runtime} than BaB with \deeppoly, a single-neuron relaxation commonly used in practice.

\begin{figure*}
  \centering
  \resizebox{\linewidth}{!}{
  \begin{tikzpicture}
    \def\scalemath{0.45}
    \def\opacity{0.8}

	\def\a{0.4}
  
	\def\b{1.4}
	\def\al{0.15}
	\coordinate (a) at ({-0.4},{0.4});
    \coordinate (b) at ({0.4},{0.4});
    \coordinate (c) at ({0.4},{-0.4});
    \coordinate (d) at ({-\a},{-\a});
    \coordinate (e) at ({-0.4},{0.2});
    \coordinate (f) at ({-0.1},{0.4});

	\fill[thick, my-blue!90]  (f) -- (b) -- (c) -- (d) -- (e) -- cycle;
    \fill[thick, cred!90]  (a) -- (f) -- (e) -- cycle;
	
    \node at (1.7,0.0) {$\Rightarrow$};
    \node[scale =\scalemath] at (1.7,0.25) {$\vy = \mA_1\vx + \vb_1$};

    \coordinate (a1) at ({2.48},{-0.2});
    \coordinate (b1) at ({2.88},{0.6});
    \coordinate (c1) at  ({3.28},{0.2});
    \coordinate (d1) at ({2.88},{-0.6});
    \coordinate (e1) at ({3},{0.48});
    \coordinate (f1) at ({2.725},{0.3});

	\fill[fill=my-blue!90] (a1) -- (f1) -- (e1) -- (c1) -- (d1)-- cycle;
    \fill[fill=cred!90] (f1) -- (e1) -- (b1) -- cycle;

    \node at (4.7,0.0) {$\Rightarrow$};
    \node[scale =\scalemath] at (4.7,0.25) {$\vz = \rho(\vy)$};

    \coordinate (a2) at ({6},{0});
    \coordinate (b2) at ({6},{0.48});
    \coordinate (c2) at  ({6.28},{0.2});
    \coordinate (d2) at ({6.18},{0});
    \coordinate (e2) at ({6},{0.3});
    \coordinate (f2) at ({6},{0.6});
    \fill[my-blue!90]  (a2) -- (b2) -- (c2) -- (d2)-- cycle;
    \draw[ultra thick, color=cred!90]  (e2) -- (f2);
    \node at (7.7,0.0) {$\Rightarrow$};
    \node[scale = 0.5] at (7.7,0.25) {$\vv = \mA_2 \vz +\vb_2$};

    \coordinate (a3) at ({9},{-0.4});
    \coordinate (b3) at ({8.52},{0.08});
    \coordinate (c3) at  ({9.08},{0.08});
    \coordinate (d3) at ({9.18},{-0.22});
    \coordinate (e3) at ({8.7},{-0.1});
    \coordinate (f3) at ({8.4},{0.2});
    \fill[fill=my-blue!90] (a3) -- (b3) -- (c3) -- (d3)  -- cycle;
    \draw[thick, cred!90]  (e3) -- (f3);

    \draw[-Stealth] (-0.8, 0) -- (0.8, 0) node[right,scale=\scalemath] {$x_1$};
	\draw[-Stealth] (0, -0.8) -- (0,0.8) node[above,scale=\scalemath] {$x_2$};
    \draw[-Stealth] (3-0.8, 0) -- (3+0.8, 0) node[right,scale=\scalemath] {$y_1$};
	\draw[-Stealth] (3, -0.8) -- (3, 0.8) node[above,scale=\scalemath] {$y_2$};
    \draw[-Stealth] (6-0.8, 0) -- (6+0.8, 0) node[right,scale=\scalemath] {$z_1$};
	\draw[-Stealth] (6, -0.8) -- (6, 0.8) node[above,scale=\scalemath] {$z_2$};
    \draw[-Stealth] (9-0.8, 0) -- (9+0.8, 0) node[right,scale=\scalemath] {$v_1$};
	\draw[-Stealth] (9, -0.8) -- (9, 0.8) node[above,scale=\scalemath] {$v_2$};
\end{tikzpicture}
  }
  \caption{A partition of the input set where every part remains as a convex polytope through the layers.}
  \label{fig:polytope_preserve}
\end{figure*}

\section{Discussion} \label{sec:discussion}

\paragraph{Result Contexts}
Our results (both negative and positive) are limited to the neural network verification setting considered in this paper. 
Under the given setting, the \emph{universal convex barrier} (\cref{sec:no_cross_layer} and \cref{sec:cross_layer}) means that for any fixed convex relaxation with arbitrarily large but finite resources, any non-empty input convex polytope, and any error threshold, there exists a ReLU network whose bounding error exceeds that threshold. The absolute-error result also generalizes to common non-ReLU activations satisfying the conditions in \cref{app:non-relu}, including logistic sigmoid and $\tanh$. For ReLU networks, both the absolute and relative bounding errors can be arbitrarily large. The positive results (\cref{sec:augmented_multi_neuron}) are further limited to finite ReLU networks.

\paragraph{Implications}
We established a universal convex barrier, 
essentially ruling out the possibility of complete verifiers based solely on any convex relaxation. 
This implies that convex relaxations should only be applied as a subroutine in a complete verification method, such as BaB. 
All existing BaB methods apply single-neuron relaxations 
 for bounding the subproblems. However, our results suggest that 
  subproblem bounding with multi-neuron relaxations 
has strictly lower partition complexity.
 This indicates potential interest in applying efficient multi-neuron relaxations to bound the subproblems during BaB.  
  In addition, existing efforts on training certified models focus on single-neuron relaxations, despite the fact that none of the single-neuron relaxations can provide exact bounds for any networks encoding complex functions. 
In contrast,
 results established in \cref{sec:equivalent_transformation} suggest that certified training with multi-neuron relaxations may be more effective, 
  as they can provide exact bounds for every \PL function encoded by some networks. We leave the further investigation of practical algorithms to future work.

\section{Conclusion}
We conducted the first in-depth study on the expressiveness of multi-neuron convex relaxations. 
We extended the established single-neuron convex barrier to a \textit{universal convex barrier} for multi-neuron relaxations, 
showing that they are inherently incomplete regardless of the resources allocated. On the positive side, we showed that 
completeness can be achieved by 
 multi-neuron relaxations
when augmented with equivalency-preserving network transformations or convex polytope partitioning, and established clear separations between multi-neuron and single-neuron relaxations in both cases. 
Our findings lay a solid  foundation for multi-neuron relaxations and point to new directions for certified robustness.

\clearpage
\setlength{\textfloatsep}{\savedtextfloatsep}
\setlength{\intextsep}{\savedintextsep}
\flushbottom
\section*{Acknowledgments}

This work has been done as part of the EU grant ELSA (European Lighthouse on Secure and Safe AI, grant agreement no. 101070617) and the SERI (Swiss State Secretariat for Education, Research and Innovation) grant SAFEAI (Certified Safe, Fair and Robust Artificial Intelligence, contract no. MB22.00088). Views and opinions expressed are however those of the authors only and do not necessarily reflect those of the European Union or European Commission. Neither the European Union nor the European Commission can be held responsible for them. 

\message{^^JLASTBODYPAGE \thepage^^J}

\section*{Proof Formalization}
All theoretical results and proofs are formalized in Lean using Paper Formalizer~\citep{mao2026paperformalizer}, and are available at \url{https://github.com/AlgebraLoveme/multineuron-formalization}.

\bibliography{references}
\bibliographystyle{iclr2026_conference}

\appendix

\section{Related Work} \label{sec:related_work}

\paragraph{Neural Network Certification} 
We focus on the deterministic certification of neural networks, which provides a deterministic guarantee on the robustness of a model rather than a probabilistic one \citep{CohenRK19,LecuyerAG0J19,SalmanSYKK20, carlinicertified,chenhao25acr,sun2025dual}.
Existing methods for deterministic neural network certification can be categorized into complete and incomplete methods. 
Complete methods commonly rely on solving a mixed-integer program \citep{TjengXT19,AndersonHMTV20,TjandraatmadjaA20,TsayKTM21}
to provide exact bounds for the output of a network.  The state-of-the-art complete method \citep{ZhangWXLLJ22,ShiJKZJH24,XuZ0WJLH21,FerrariMJV22} is based on solving the mixed integer program with branch-and-bound \citep{BunelLTTKK20} on the integer variables.
These methods are naturally computationally expensive and do not scale well. Incomplete methods, on the other hand, provide sound but inexact bounds, based on convex relaxations of the feasible output set of a network. \citet{XuS0WCHKLH20} characterizes widely-recognized single-neuron convex relaxations \citep{MirmanGV18,WongSMK18,ZhangWCHD18,SinghGPV19} 
by their induced affine constraints, where the bounds are yielded by efficient but not necessarily optimal solvers. However, \citet{SalmanY0HZ19} empirically identify a single-neuron convex barrier, preventing 
single-neuron relaxations from providing exact bounds for general ReLU networks, even with costly optimal solvers. To bypass this barrier, multi-neuron 
relaxations \citep{SinghGMPV18,ZhangWXLLJ22,MullerMSPV22} have been proposed and achieved higher precision empirically.

\paragraph{Multi-neuron Relaxations in Practice} 
To bypass the single-neuron barrier, multi-neuron 
relaxations \citep{SinghGMPV18,ZhangWXLLJ22,MullerMSPV22} have been proposed, achieving higher precision empirically.
In particular, \cite{SinghGPV19B} and \citet{MullerMSPV22} are 
looser versions of $\gP_1$ discussed in this paper; \citet{ZhangWXLLJ22} is a looser version of $\gP_L$.
 \citet{FerrariMJV22} combine multi-neuron relaxations with BaB and find that applying multi-neuron relaxations before BaB yields a superior overall performance. 
 These practical applications motivate us to rigorously study the fundamental limit of multi-neuron relaxations. 
 Furthermore,  
the certified training community \citep{MullerE0V23, MaoM0V23, MaoBV24}  has already employed multi-neuron relaxations in verification, but not yet in training. 
This also motivates us to explore the possibility of combining multi-neuron with certified training.

\paragraph{Certification with Convex Relaxations} 
Existing work on the certification with convex relaxations focuses on the expressiveness of single-neuron relaxations. We distinguish three convex relaxation methods typically considered by theoretical work: Interval Bound Propagation (\ibp) \citep{MirmanGV18,GowalDSBQUAMK18}, which ignores the interdependency between neurons and use intervals $\{[a,b] \mid a,b \in \sR\}$ for relaxation; 
Triangle relaxation \citep{WongK18}, which approximates the ReLU function by a triangle in the input-output space; 
and multi-neuron relaxations \citep{SinghGMPV18,ZhangWXLLJ22,MullerMSPV22} which considers a group of ReLU neurons jointly in a single affine constraint.
On the positive side,
\citet{BaaderMV20} show the universal approximation theorem for certified models,
stating that for every %
\PL function $f: \sR^n \rightarrow \sR$ and every $\epsilon>0$,
there exists a ReLU network that approximates $f$, for which %
\ibp
 provides bounds within error $\epsilon$.
This result is generalized to other activations by \citet{WangAPJ22}.
However, \citet{MirmanBV22} shows that there exists a \PL function
for which \ibp analysis of every finite ReLU network encoding this function provides inexact bounds.
Further, \citet{MaoMFV24} shows that a strong regularization on the parameter signs is required for \ibp to provide good bounds.
Beyond \ibp, \citet{baader2024expressivity} show that even Triangle,
the most precise single-neuron relaxation,
cannot exactly bound any ReLU network that encodes the ``$\max$''
function in $\sR^2$,
although it is provably more expressive than \ibp in $\sR$. While \citet{baader2024expressivity} also shows that every ReLU network with a single hidden layer can be exactly bounded by multi-neuron relaxations with sufficient budget, the theoretical properties of multi-neuron relaxations in the certification of general ReLU networks remain unknown. We remark that this review is not exhaustive, especially regarding convex relaxations beyond neural network certification, and refer readers to \citet{HuchetteMST23} for a more comprehensive survey on MILP formulations, polyhedral geometry and expressiveness of ReLU networks.

\section{Notation} \label{app:notation}
We use lowercase boldface letters to denote vectors and uppercase boldface letters to denote matrices.
For the vector $\vx$, $\vx_i$ denotes its $i$-th entry and $\vx_I$ is
the subvector of $x$ with entries corresponding to the indices in the set $I$.
$I_N$ is the $N\times N$ 
identity matrix and $\mathbf{1}_N$ and $\mathbf{0}_N$  denotes the $N$-dimensional column vector with all entries equal to $1$ and $0$, respectively.
$\ve_i$ is the column vector with the $i$-th element taking value $1$ and all other elements $0$.
For $\mathbf{1}$ and $\mathbf{0}$ without subscript, we understand them to be vectors of appropriate dimensions according to the context. 
For matrices $\mA_1,\ldots,\mA_n$, we designate
the block-diagonal matrix with diagonal element-matrices  $\mA_1,\ldots,\mA_n$, by $\diag(\mA_1,\ldots,\mA_n)$.

For a set $H$, we denote the convex hull of $H$ by $\conv H$. We represent the ReLU function as $\rho(\vx) = \max(\vx, \vzero)$.
For two vectors $\va,\vb\in \sR^d, \va \leq \vb$ denotes elementwise inequality.

The real set is denoted by $\sR$, the natural numbers by $\sN$, the positive integers by $\sN^+$, and the $d$-dimensional real space by $\sR^d$. For a set $S$, $|S|$ denotes the cardinality of $S$, which is the number of elements in $S$. Given $r \in \sN^+$, $[r]$ denotes the set $\{1, 2, \ldots, r\}$. For a function $f$ and input domain $X$, we use $f(X)$ to denote its range $\{f(\vx) \mid \vx \in X\}$ and $f[X]$ to denote its image $\{ (\vx, f(\vx)) \mid \vx \in X\}$.

We use $\gC(\vx)$ to denote a set of affine constraints on $\vx$, i.e., $\gC = \{ \mA \vx + \vb \le \vzero\}$ for some matrix $\mA$ and some vector $\vb$. For two sets of constraints $\gC_1(\vx) = \{ \mA^{(1)} \vx + \vb^{(1)} \le \vzero\}$ and $\gC_2(\vx) = \{ \mA^{(2)} \vx + \vb^{(2)} \le \vzero\}$, $\gC_1 \land \gC_2 = \{ \mA^{(1)} \vx + \vb^{(1)} \le \vzero \land \mA^{(2)} \vx + \vb^{(2)} \le \vzero\}$ denotes a combination of the two sets of constraints, i.e., their feasible sets are intersected.

Given a set $H = \{(\vx, \vy) \mid (\vx, \vy) \in H\}$, we denote the projection of $H$ onto the $\vx$-space by $\pi_{\vx}(H) = \{ \vx \mid \exists \vy: (\vx, \vy) \in H\}$ and the projection onto the $\vy$-space by $\pi_{\vy}(H) = \{ \vy \mid \exists \vx: (\vx, \vy) \in H\}$. For a feasible set $\gC$ defined by the constraint set $\gC(\vx, \vy)$, $\pi_{\vx}(\gC)$ is the set of values of $\vx$ that satisfy the constraints in $\gC$.

For a function $f: \sR^{\din} \rightarrow \sR$, an input convex polytope $X \in \sR^{\din}$ and a convex relaxation $\gP$, the lower bound of $f$ on $X$ under $\gP$ is denoted by $\ell(f, \gP, X)$ and the upper bound is denoted by $u(f, \gP, X)$. Concretely, let $\gC(\gP)$ be the constraint set induced by $\gP$ and $v \in \sR$ be the output variable, then $\ell(f, \gP, X) = \min \pi_{v}(\gC(\gP))$ and $u(f, \gP, X) = \max \pi_{v}(\gC(\gP))$.

We call neurons that switch their activation states within the input set as unstable, otherwise call it stable.

\section{Example Illustration} \label{app:illustration}

\begin{figure*}
	\centering
	\resizebox{\linewidth}{!}{
	\begin{tikzpicture}
    \tikzset{comp_node/.style={circle,draw=black,fill=white,inner sep=0pt,minimum size=13pt}}
    \def\xsep{2}
    \def\ysep{1}
    \begin{scope}
        \node[comp_node] (x) at (0,0) {$x$};
        \node[comp_node] (a) at ($(x) + (\xsep,\ysep)$) {$a$};
        \node[comp_node] (b) at  ($(x) + (\xsep,-\ysep)$) {$b$};
        \draw[-Stealth] (x) -- (a) node[midway,above] {$1$};
        \draw[-Stealth] (x) -- (b) node[midway,below] {$1$};
        \node[comp_node] (c) at ($(a) + (\xsep,0)$) {$c$};
        \node[comp_node] (d) at ($(b) + (\xsep,0)$) {$d$};
        \draw[-Stealth] (a) -- (c) node[midway,above] {$\rho$};
        \draw[-Stealth] (b) -- (d) node[midway,below] {$\rho$};
        \node[comp_node] (f) at ($(c) + (\xsep,-\ysep)$) {$f$};
        \draw[-Stealth] (c) -- (f) node[midway,above] {$1$};
        \draw[-Stealth] (d) -- (f) node[midway,below] {$-1$};
    \end{scope}

    \begin{scope}
        \node (left_text) at ($(x)+(-4,0)$) {
            $ \gC_s(a,c) = \left\{
            \begin{bmatrix}
                a-c \\
                -c \\
                c-\frac{1}{2}(a+1) 
            \end{bmatrix} \le 0 \right\}
            $
        };
    \end{scope}

    \begin{scope}
        \node (right_text) at ($(f)+(5,0)$) {
            $ \gC_m(a,b,c,d) = \gC_s(a,c) \cap \gC_s(b,d) \cap \left\{
            \begin{bmatrix}
                c-d \\
                a-b 
            \end{bmatrix} = 0 \right\}
            $
        };
    \end{scope}
\end{tikzpicture}
	}
	\caption{Visualization of the single-neuron and multi-neuron relaxations for a network encoding $f(x)=0$.}
	\label{fig:numeric_example}
\end{figure*}

This section contains a toy example to illustrate the concepts we introduced,
namely the ReLU network $\rho(x) - \rho(x)$ encoding the zero function $f(x) = 0$ with input $x \in [-1, 1]$. This network is visualized in \cref{fig:numeric_example}.
The affine constraints are as follows: (i) for the input convex polytope,
we have $\{x \ge -1, x \le 1\}$;
(ii) for affine layers, we have $\{a=x,b=x,f = c-d\}$;
(iii) for the ReLU layer, a single neuron relaxation (Triangle) will have $\gC_{\text{s}}(a,c)\land \gC_{\text{s}}(b,d)$, and a multi-neuron relaxation ($\gM_2$) will have $\gC_{\text{m}}(a,b,c,d)$. 
In this case, a multi-neuron relaxation successfully solves that the upper bound and lower bound of $f$ are zero, 
while a single-neuron relaxation solves an inexact upper bound $1$ and an inexact lower bound $-1$.

\section{Pseudo-algorithm for Polytope Partition} \label{app:polytope_partition}

In this section, we present a pseudo-algorithm for the polytope partition in \cref{sec:input_split}.
It serves as a high-level description of the polytope partitioning algorithm.
The actual implementation in practice may vary depending on the specific problem and the desired performance.

\begin{algorithm}
  \caption{Polytope Partition}
  \label{alg:polytope_partition}
  \begin{algorithmic}
    \State {\bfseries Input:} network $f$, input convex polytope $X$
    \State {\bfseries Output:} $u = \max_{x \in X} f(x)$ and $\ell = \min_{x \in X} f(x)$

    \State Initialize $H \gets \{(X, X)\}$
    \For{each layer $f_j$ in $f$}
      \State Initialize with a convex polytope collection $H' = \emptyset$
      \For{each pair $(H_k, S_k) \in H$}
        \State Compute the output of $f_j$ on $S_k$, denoted by $f_j(S_k)$
        \If{$f_j(S_k)$ is a convex polytope}
          \State Add $(H_k, f_j(S_k))$ to $H'$
        \Else
          \State Decompose $(H_k, S_k)$ into $\nu$ convex polytopes $H_{k_1}, \ldots, H_{k_\nu}$ and the images $S_{k_1}, \ldots, S_{k_\nu}$, such that $f_j(S_{k_i})$ is a convex polytope for $i=1,\ldots,\nu$, where $\nu$ should be as small as possible
          \State Add $(H_{k_i}, f_j(S_{k_i}))$ to $H'$ for $i=1,\ldots,\nu$
        \EndIf
      \EndFor
      \State Set $H = H'$
    \EndFor
    \State Initialize $\ell = +\infty$ and $u = -\infty$
    \For{each convex polytope $H_k \in H$}
      \State Update $\ell = \min(\ell, \ell(f, \gP_1, H_k))$
      \State Update $u = \max(u, u(f, \gP_1, H_k))$
    \EndFor
    \State {\bfseries return} $u$ and $\ell$
  \end{algorithmic}
\end{algorithm}

\paragraph{Example.}
Running \cref{alg:polytope_partition} on the ``$\max$'' example in \cref{sec:equivalent_transformation}, the input box $[0,1]^d$ is always mapped to a convex polytope as it passes through 
the network layers. Therefore, the partition complexity is $1$.

We remark that there are two steps in the algorithm that might require high computational complexity in practice:
(i) the partitioning of a set into convex polytopes, and (ii) the merging of convex polytopes. 
The partitioning step is necessary because the output of a ReLU network may not be a convex polytope, 
and we need to partition it into smaller convex polytopes to compute the bounds. 
The merging step is to merge redundant convex polytopes to reduce the number of subproblems. 
To design a practical algorithm with a low running time complexity 
is beyond the scope of this paper, 
and we leave it to the future work.

\section{Deferred Proofs in \cref{sec:no_cross_layer}}\label{app:proofs_6}
\subsection{Proof of \cref{lem:do_not_reduce_feasible_set}} \label{app:proof_lem:do_not_reduce_feasible_set}

We prove \cref{lem:do_not_reduce_feasible_set}, restated below for convenience.
\growfeasibleset*
\begin{proof}
  As $\gP$ does not consider constraints cross nonadjacent layers, $\gC_1$ is in the form of 
  $\gC(\vx, \vv^{(1)}) \cup \gC(\vv^{(1)}, \vv^{(2)}) \cup \cdots \cup \gC(\vv^{(i-1)}, \vv^{(i)})$ and $\gC_2 = \gC_1 \cup \gC(\vv^{(i)}, \vv^{(i+1)}) \cup \cdots \cup \gC(\vv^{(L-1)}, \vv^{(L)})$.
  Let $\gC_3 := \gC(\vv^{(i)}, \vv^{(i+1)}) \cup \cdots \cup \gC(\vv^{(L-1)}, \vv^{(L)})$.
  Note that the projection $\pi_{\vv^{(i)}}(\gC_1)$ is considered by $\gP_1$ as the input set of the subnetwork $f_{i+1}\circ \cdots \circ f_L$ to instantiate further relaxations for deeper layers. 
  Since $\gP_1$ is a sound verifier, the constraints $\gC_3$ must allow the input set, i.e.,
  \[
      \pi_{\vv^{(i)}}(\gC_3) \supseteq   \pi_{\vv^{(i)}}(\gC_1). 
  \]
  Now
  $\pi_{\vv^{(i)}}(\gC_2)$ is obtained by applying the Fourier-Motzkin algorithm to eliminate all the variables in $\gC_2 = \gC_1\cap \gC_3$ except $\vv^{(i)}$.
  W.l.o.g, assume we eliminate in the following order $\vx,\vv^1,\ldots, \vv^{(i-1)}$, $\vv^{(i+1)},\ldots, \vv^{(L)}$.
  The constraints in $\gC_3$ remains unchanged as we eliminate $\vx,\vv^1,\ldots, \vv^{(i-1)}$, since they are not included in $\gC_3$.
  Therefore, 
  \[
      \pi_{\vv^{(i)}}(\gC_2) = \pi_{\vv^{(i)}}(\gC_1)\cap \pi_{\vv^{(i)}}(\gC_3).
  \]
  Hence, $\pi_{\vv^{(i)}}(\gC_2) =\pi_{\vv^{(i)}}(\gC_1)$.
\end{proof}

\subsection{Proof of \cref{lem:bound_cannot_improve}} \label{app:proof_lem:bound_cannot_improve}

We prove \cref{lem:bound_cannot_improve}, restated below for convenience.
\boundcannotimprove*
\begin{proof}
 By the notation in \cref{lem:do_not_reduce_feasible_set},
 \begin{align*}
  \ell(f, \gP_1, X) &= \min_{\mu \in \pi_{\vv^{(2)}}(\gC_2(\vx, \vv^{(1)}, \vv^{(2)}))} \mu \\
  &\le \min_{\nu \in \pi_{\vv^{(1)}}(\gC_2(\vx, \vv^{(1)}, \vv^{(2)}))} f_2(\nu) \\
  &= \min_{\nu \in \pi_{\vv^{(1)}}(\gC_2(\vx, \vv^{(1)}))} f_2(\nu),
\end{align*}
where the last equality follows from \cref{lem:do_not_reduce_feasible_set}.
  Since $\gC_1(\vx, \vv^{(1)})$ is a convex polytope containing the feasible set of $\vv^{(1)}$, we have $\pi_{\vv^{(1)}}(\gC_1(\vx, \vv^{(1)})) \supseteq \conv(f_1(X))$. Therefore,
  \begin{align*}
      \ell(f, \gP_1, X) &\le \min_{\nu \in \pi_{\vv^{(1)}}(\gC_2(\vx, \vv^{(1)}))} f_2(\nu) \\
      &\le \min_{\nu \in \conv(f_1(X))} f_2(\nu) \\
      &= \min (f_2 (\conv(f_1(X)))).  
  \end{align*}

  The proof for the upper bound is similar.
\end{proof}

\subsection{Proof of \cref{thm:infinite_relaxation_error}} \label{app:proof_thm:infinite_relaxation_error}

Now we prove \cref{thm:infinite_relaxation_error}, restated below for convenience.

\infiniterelaxationerror*

\begin{proof}
  The proof is done by explicit construction of ReLU networks that satisfies the required property. 

  Since $X$ is non-singleton, there exists an affine map $\pi:\sR^d\to\sR$ such that $\pi(X)=[a,b]$ with $a<b$. Let $W_0(x) = 2\frac{x-a}{b-a} -1$, $W_1(x) = (x+1, x)$, and
  \[
    f^\prime(\vx) = 2T|\vx_1 - 1| + 2T|\vx_2 - 0.5|
    = 2T\rho(\vx_1-1)+2T\rho(1-\vx_1)+2T\rho(\vx_2-0.5)+2T\rho(0.5-\vx_2),
  \]
  for $\vx\in \sR^2$.
  We construct the network as $f = f^\prime \circ \rho \circ W_1 \circ W_0\circ \pi$. Since $\rho \circ W_1 \circ W_0(a) = (0,0)$ and $\rho \circ W_1 \circ W_0(b) = (2,1)$,
  \[
    \conv(\rho \circ W_1 \circ W_0([a,b])) \supseteq \{(2t, t) \mid t\in [0,1]\}.
  \]
  Thus,
  \[
    \min f^\prime(\conv(\rho \circ W_1 \circ W_0([a,b]))) = 0.
  \]
  Therefore, by \cref{lem:bound_cannot_improve}, $\ell(f,\gP_1,X) \leq 0$. However, for every $\vx\in X$, the point $(\rho \circ W_1 \circ W_0\circ \pi)(\vx)$ lies on the two-segment set joining $(0,0)$ to $(1,0)$ and $(1,0)$ to $(2,1)$, on which $f^\prime$ has minimum $T$. Hence $\min f(X)=T$, and $\ell(f,\gP_1,X)\leq \min f(X)-T$.

  For the upper-bound statement, define $g:=-f$ by negating the final affine layer. The layerwise relaxation of $g$ is obtained from the layerwise relaxation of $f$ by the affine change of the output variable $y_g=-y_f$. Hence,
  \[
    u(g,\gP_1,X)=-\ell(f,\gP_1,X)\geq 0,
    \qquad
    \max g(X)=-\min f(X)=-T.
  \]
  Therefore $u(g,\gP_1,X)\geq \max g(X)+T$.

\end{proof}

\section{Deferred Proofs in \cref{sec:cross_layer}}
\subsection{Proof of \cref{lem:pumping_lemma}} \label{app:proof_lem:pumping_lemma}

Now we prove \cref{lem:pumping_lemma}, restated below for convenience.
\pumpinglemma*
\begin{proof}
  Intuitively, the proof is done by blocking direct information passing from $f_1$ to $f_2$ through adding dummy layers. 
  Take
\begin{equation}\label{eq:L}
  L = \lceil \max(\frac{1}{\alpha}, \frac{L_1+L_2+1}{1-\alpha}) \rceil
\end{equation}
and set $r = \max(1, \lfloor \alpha L\rfloor)$.
We construct the network $f$ by pumping $f_2 \circ f_1$ through adding identity layers between $f_2$ and $f_1$,
  thus the name pumping lemma. Concretely, let $f= f_{2}\circ\hspace{-0.2cm} \underbrace{I_{d'} \circ \cdots \circ I_{d'}}_{(L-L_1 - L_2) \text{ times}} \hspace{-0.2cm} \circ f_1 $, where $I_{d'}$ is the identity function in $\sR^{d'}$. By \eqref{eq:L}, we have $L-L_1-L_2>\alpha L\geq r-1$, and hence $L-L_1-L_2\geq r$ since the left-hand side is an integer.
Denote the input variable by $\vv^{(0)}$ and the
 variables on the $i$-th layer of $f$ by $\vv^{(i)}$. By definition, $\gP_r$ computes all constraints of the form $\gC(\vv^{(i)},\ldots,\vv^{(i+r)})$ for $i=0,\ldots,L-r$. 
By the identity layer construction, 
 we know $\vv^{(L_1)}=\vv^{(L_1+1)}=\ldots=\vv^{(L-L_2)}$. 
 Since $L - L_1 - L_2 \ge r$, 
 the constraints induced by $\gP_r$ can be reduced to 
constraints of the form $\gC(\vv^{(i)},\ldots,\vv^{(\min(i+r, L_1))})$, for $i=0,\ldots,L_1$, 
 and $\gC(\vv^{(\max(j-r, L-L_2))},\ldots,\vv^{(j)})$, for $j=L-L_2,\ldots,L$. 
For brevity, we slightly abuse notation and denote by $\gC(\gP_r)$ the union of all constraints induced by $\gP_r$,
denote by $\gC_1$
the union of constraint sets of the form $\gC(\vv^{(i)},\ldots,\vv^{(\min(i+r, L_1))})$ for $i=0,\ldots,L_1$,
and denote by $\gC_2$
the union of constraint sets of the form $\gC(\vv^{(\max(j-r, L-L_2))},\ldots,\vv^{(j)})$ for $j=L-L_2,\ldots,L$.
 Thus, $\pi_{\vv^{(L-L_2)}}(\gC(\gP_r)) = \pi_{\vv^{(L_1)}}(\gC(\gP_r)) = \pi_{\vv^{(L_1)}}(\gC_1)$. 
 Since $\conv(f_1(X)) \subseteq \pi_{\vv^{(L_1)}}(\gC_1)$, 
  \begin{align*}
    \ell(f, \gP_{r}, X) &\le \min f_2 (\pi_{\vv^{(L-L_2)}}(\gC(\gP_r))) \\
    &= \min f_2(\pi_{\vv^{(L_1)}}(\gC_1)) \\
    &\le \min f_2(\conv(f_1(X))),
  \end{align*}
  and 
  \begin{align*}
    u(f, \gP_{r}, X) &\ge \max f_2 (\pi_{\vv^{(L-L_2)}}(\gC(\gP_r))) \\
    &= \max f_2(\pi_{\vv^{(L_1)}}(\gC_1)) \\
    &\ge \max f_2(\conv(f_1(X))).
  \end{align*}
\end{proof}

\subsection{Proof of \cref{thm:cross_layer}} \label{app:proof_thm:cross_layer}

Now we prove \cref{thm:cross_layer}, restated below for convenience.
\crosslayerincompleteness*
\begin{proof}
  We reuse the construction in the proof of \cref{thm:infinite_relaxation_error}, augmented by \cref{lem:pumping_lemma}.
  Let $f_1:=\rho\circ W_1\circ W_0\circ\pi$ and $f_2:=f^\prime$ be the subnetworks from that construction, and write $\hat{f}:=f_2\circ f_1$.
  The construction gives $\min\hat{f}(X)=T$ and $\min f_2(\conv(f_1(X)))=0$.
  By \cref{lem:pumping_lemma}, there exist a depth $L$ and an $L$-layer network $f$ such that $f=\hat{f}$ on $X$.
  Set $r:=\max(1,\lfloor\alpha L\rfloor)$. The lemma gives
  \[
    \ell(f,\gP_r,X)\leq\min f_2(\conv(f_1(X)))=0=\min f(X)-T.
  \]

  For the upper-bound statement, define $g:=-f$ by negating the final affine layer.
  This preserves the depth $L$ and hence the same value of $r$.
  By the definition of $\gP_r$, windows omitting the output variable are unchanged, while the exact graph tuples of windows containing it are transformed by reflection of that coordinate.
  This invertible linear reflection commutes with convex hulls and intersections and leaves projections onto preceding variables unchanged.
  Thus the full relaxed feasible sets correspond by output reflection, giving
  \[
    u(g,\gP_r,X)=-\ell(f,\gP_r,X)
    \geq-\min f(X)+T
    =\max g(X)+T.
  \]
\end{proof}

\section{Deferred Proofs in \cref{sec:augmented_multi_neuron}}\label{app:proofs_5}

\subsection{Proof of \cref{thm:precise_bounds} and \cref{prop:multi_neuron_expressiveness}}\label{app:proof_thm:precise_bounds}

Now we prove \cref{thm:precise_bounds}, restated below for convenience.
\precisebounds*
\begin{proof}
  Let $\vm:=\min f(X)$, and $\vn:=\max f(X).$
 Since $X$ is a compact polytope and $f$ is continuous, these extrema are finite and are attained.
  For $\vy\in\sR^{d'}$, define
  \[
    \vc(\vy):=\vm+\rho(\vy-\vm)-\rho(\vy-\vn).
  \]
  Thus $\vc(\vy)\in[\vm,\vn]$ for every $\vy\in\sR^{d'}$, and $\vc(\vy)=\vy$ whenever $\vy\in[\vm,\vn]$. The map $\vc$ is realized by one affine layer computing $(\vy-\vm,\vy-\vn)$, followed by a ReLU layer and a final affine layer, as shown coordinatewise in \cref{fig:clipping_network}. We set
  \[
    g := \vc\circ f .
  \]
  Since $f(X) \subseteq [\vm,\vn]$, $g=f$ on $X$.
\begin{figure}[H]
	\centering
	\resizebox{.55\linewidth}{!}{
	\begin{tikzpicture}
    \tikzset{comp_node/.style={circle,draw=black,fill=white,inner sep=0pt,minimum size=15pt}}
    \def\xsep{2}
    \def\ysep{1}
    \begin{scope}
        \node[comp_node] (y) at (0,-0.5*\ysep) {$\vy_i$};

        \node[comp_node] (a) at ($(y) + (\xsep,0.5*\ysep)$) {$\va_i$};
        \node[comp_node] (b) at ($(y) + (\xsep,-0.5*\ysep)$) {$\vb_i$};
        \draw[-Stealth] (y) -- (a) node[midway,above] {$1$};
        \draw[-Stealth] (y) -- (b) node[midway,below] {$1$};
        \node[above=1pt of a] {\scriptsize $-\vm_i$};
        \node[below=1pt of b] {\scriptsize $-\vn_i$};

        \node[comp_node] (p) at ($(a) + (\xsep,0)$) {$\vp_i$};
        \node[comp_node] (q) at ($(b) + (\xsep,0)$) {$\vq_i$};
        \draw[-Stealth] (a) -- (p) node[midway,above] {$\rho$};
        \draw[-Stealth] (b) -- (q) node[midway,below] {$\rho$};

        \node[comp_node] (z) at ($(p) + (\xsep,-0.5*\ysep)$) {$\vz_i$};
        \draw[-Stealth] (p) -- (z) node[midway,above] {$1$};
        \draw[-Stealth] (q) -- (z) node[midway,below] {$-1$};
        \node[right=2pt of z] {\scriptsize $+\vm_i$};
    \end{scope}
\end{tikzpicture}
	}
	\caption{One coordinate of the clipping subnetwork realizing $\vc_i(\vy)=\vm_i+\rho(\vy_i-\vm_i)-\rho(\vy_i-\vn_i)$.}
	\label{fig:clipping_network}
\end{figure}

  It remains to show that $\gP_1$ returns exact coordinatewise bounds for $g$ on $X$. The key point is that we do not need any tightness from the $\gP_1$ relaxation of the original network $f$; in fact, we do not even need to know $\ell(f, \gP_1, X)$ and $u(f, \gP_1, X)$, other than $\ell(f, \gP_1, X) \le \vm$ and $u(f, \gP_1, X) \ge \vn$. We only need to show that $\gP_1$ is exact for the three layers realizing $\vc$; then since $\vc(\vy) \in [\vm,\vn]$ for any $\vy$, the bounds for $g$ on $X$ are exactly $\vm$ and $\vn$.

  Indeed, we use the variables shown in \cref{fig:clipping_network}. Let $S_y$ be the current projected feasible set of the relaxed variable $\vy$ by $\gP_1$. The first affine layer in \cref{fig:clipping_network} is represented exactly by $\gP_1$, so after this layer the feasible set of the variables $(\va,\vb)$ is
  \[
    S_{ab}:=\{(\va,\vb): \exists \vy\in S_y,
    \va=\vy-\vm,\; \vb=\vy-\vn\}.
  \]
  The variables $\vp$ and $\vq$ are then the ReLU outputs corresponding to $\va$ and $\vb$. By the definition of $\gP_1$, the constraint set imposed on this ReLU layer is exactly $\conv(S_{pq})$, where
  \[
    S_{pq} = \{(\va,\vb,\vp,\vq): (\va,\vb)\in S_{ab},\;
    \vp=\rho(\va),\; \vq=\rho(\vb)\}.
  \]
  For every point in $S_{pq}$ and every coordinate $i$, we have
  \[
    \va_i=\vy_i-\vm_i,
    \qquad
    \vb_i=\vy_i-\vn_i
  \]
  for some $\vy\in S_y$. Hence $\va_i\geq \vb_i$ and $\va_i-\vb_i=\vn_i-\vm_i \ge 0$. Since $\rho$ is monotone and 1-Lipschitz, every point in the exact graph satisfies
  \[
    0\leq \vp_i-\vq_i
    =\rho(\va_i)-\rho(\vb_i)
    \leq \va_i-\vb_i
    =\vn_i-\vm_i.
  \]
  These are affine inequalities in $(\va,\vb,\vp,\vq)$, and therefore remain valid on the $\conv(S_{pq})$. Thus, $\gP_1$ imposes the following constraints on $(\vp,\vq)$:
  \[
    0\leq \vp-\vq\leq \vn-\vm.
  \]

  The final affine layer in \cref{fig:clipping_network} is represented exactly by $\gP_1$, so the relaxed output variable $\vz$ satisfies
  \[
    \vz=\vm+\vp-\vq.
  \]
  Consequently, we have $\vm\leq \vz\leq \vn$. Therefore,
  \[
    \ell(g,\gP_1,X)\geq \vm,
    \qquad
    u(g,\gP_1,X)\leq \vn.
  \]

  Conversely, $\gP_1$ is sound and $g(X)=f(X)$ on $X$. Since the coordinatewise extrema $\vm$ and $\vn$ are attained by points in $f(X)=g(X)$, soundness gives
  \[
    \ell(g,\gP_1,X)\leq \vm,
    \qquad
    u(g,\gP_1,X)\geq \vn.
  \]
  Combining the two inequalities yields
  \[
    \ell(g,\gP_1,X)=\min f(X),
    \qquad
    u(g,\gP_1,X)=\max f(X).
  \]
\end{proof}

  We proceed to prove \cref{prop:multi_neuron_expressiveness}, restated below for convenience.
  \multineuronexpressiveness*

  \begin{proof}
   For a continuous piecewise linear function $f:\sR^d \rightarrow \sR$, by Theorem 2.1 of \cite{aroraUNDERSTANDINGDEEPNEURAL2018b}, there exists a ReLU network $g':\sR^d \rightarrow \sR$ satisfying 
   \begin{equation}\label{eq:1}
    f(\vx) = g'(\vx), \quad \vx \in \sR^d.
   \end{equation}
  
   By \cref{thm:precise_bounds}, there exists another ReLU network $g:\sR^d\rightarrow \sR$ satisfying 
   \begin{equation}\label{eq:2}
    g(\vx) = g'(\vx), \quad \vx \in X,
   \end{equation}
   and 
   \begin{align*}
   &\ell(g, \gP_1, X) = \min g'(X) \\
    &u(g, \gP_1, X) = \max g'(X).
   \end{align*}
Combining \eqref{eq:1} and \eqref{eq:2}, we get
\[
  g(\vx) = f(\vx), \quad \vx \in X,
 \]
 and 
 \begin{align*}
 &\ell(g, \gP_1, X) = \min f(X) \\
  &u(g, \gP_1, X) = \max f(X).
 \end{align*}
  \end{proof}

\subsection{Proof of \cref{prop:convex_polytope_prop} and \cref{prop:partition}}\label{proof_prop:single_polytope_preciseness}

The following lemma gives a sufficient condition for $\gP_1$ to return exact bounds on a single convex polytope.

\begin{restatable}{lem}{singlepolytopepreciseness}\label{prop:single_polytope_preciseness}
  Let $L\in\sN^+$, and consider a network $f=f_L\circ\cdots\circ f_1$, where each $f_j$ is either an affine transformation or the ReLU function.
  Let $X$ be a convex input polytope, and denote the network prefixes by $f^{(j)}:=f_j\circ\cdots\circ f_1$ for $j\in[L]$.
  If $f^{(j)}(X)$ is a convex polytope for every $j\in[L]$, then
  $
    \ell(f,\gP_1,X)=\min f(X)
  $
  and
  $
    u(f,\gP_1,X)=\max f(X).
  $
\end{restatable}
\begin{proof}
  Let $S_0=X$ and $S_j=f^{(j)}(X)$ for $j\in[L]$.
  We show by induction that the feasible set computed by $\gP_1$ after layer $j$ is exactly $S_j$.
  The input feasible set is $S_0$.
  If the feasible set before layer $j$ is $S_{j-1}$, then $\gP_1$ constructs the convex hull of the graph of $f_j$ on $S_{j-1}$.
  Projecting this convex hull onto the output coordinates gives $\conv(f_j(S_{j-1}))=\conv(S_j)$.
  By assumption, $S_j$ is convex, so this projection equals $S_j$.
  The output feasible set is therefore exactly $S_L=f(X)$, which establishes both bound equalities.
\end{proof}

We apply this lemma to each part of the partition to prove \cref{prop:convex_polytope_prop}, restated below.

\convexpolytopeprop*
\begin{proof}
  By \cref{prop:single_polytope_preciseness}, $\gP_1$ returns exact bounds on $f(H_k)$ for every $k\in[\nu]$.
  Since $f(X)=\bigcup_{k=1}^{\nu}f(H_k)$, taking the coordinatewise minimum of these lower bounds and maximum of these upper bounds gives the exact bounds on $f(X)$.
\end{proof}

We now prove \cref{prop:partition} by extending any $\gP_1$ branching tree into a complete single-neuron branching tree.

\partition*
\begin{proof}
  On a subproblem where all neurons are stable, every ReLU acts as either the identity or the zero map, so every network prefix is affine.
  Choosing compatible activation signs that retain the branch constraints expresses the input region as the intersection of $X$ with affine halfspaces.
  This region and all its prefix images are therefore convex polytopes.
  By \cref{prop:single_polytope_preciseness}, the $\gP_1$ stopping condition is satisfied whenever all neurons are stable.

  Fix any admissible execution of $\text{BaB}(\gP_1)$, with branching tree $T$.
  Every internal node of $T$ has an unstable neuron, since otherwise $\gP_1$ would close it.
  Thus, the split-all procedure $\text{BaB}(\gS)$ can reproduce every split in $T$ without closing an internal node prematurely.
  At each leaf, it continues splitting until all neurons are stable.
  Each split makes the selected neuron stable in both children, so this extension terminates.
  It produces an admissible single-neuron execution tree $T_{\gS}$ with at least as many terminal nodes as $T$.

  Writing $N(T)$ for the number of terminal nodes of a tree $T$, we obtain
  \[
    N(T)\leq N(T_{\gS})\leq\text{\#Partition}(\gS,f,X).
  \]
  Taking the maximum over all admissible $\gP_1$ branching orders gives
  \[
    \text{\#Partition}(\gP_1,f,X)
    \leq\text{\#Partition}(\gS,f,X).
  \]
\end{proof}

\section{An Example of the Benefit of Polytope Partition} \label{app:example_partition}

For the network encoding $\max(x_1,\ldots,x_d)$ in the case study of \cref{sec:equivalent_transformation}, first note that it has $2^{d-1}$ distinct activation patterns on $[0,1]^d$.
 We show that BaB requires $2^{d-1}$
 branching to return precise bounds. Let $y_i=\max(x_1,\ldots,x_i)$, for $i \in [d-1]$, where $y_1 = x_1$.
 The $i$-th unstable neuron can then be rewritten as $\rho(y_i - x_{i+1})$, \eg for node $c$ in \cref{fig:max_case_study} which is the first unstable neuron, it can be rewritten as $\rho(y_1 - x_2)$. 
 After a branching on it, this node plus $x_{i+1}$ becomes either $x_{i+1}$ when $x_{i+1} \ge y_i$, 
 or $y_i$ when $x_{i+1}<y_i$. 
 Therefore, this branching makes two subproblems, 
 which are essentially the $(d-1)$-dimension ``$\max$'' function. 
 This directly implies that neither of the two subproblems can 
 be precisely bounded by any single-neuron relaxation, thus the branching will not stop. 
 Repeating this, BaB enumerates all $2^{d-1}$ branches, confirming the lower bound established in \cref{prop:partition}. In contrast, $\gP_1$ has partition complexity $1$ as shown in \cref{sec:equivalent_transformation}, leading to an exponential reduction.

 Regarding the runtime, note that the number of constraints introduced by $\gP_1$ grows linearly with $d$, while the number of branching grows exponentially with $d$ for BaB with DeepPoly. Thus, for this example, the runtime of $\gP_1$ grows polynomially with $d$, while that of BaB with DeepPoly grows exponentially with $d$.

\section{Relative bounding error} \label{app:relative_error}
\cref{thm:infinite_relaxation_error} and \cref{thm:cross_layer} state that the absolute bounding error 
by layerwise and cross-layer relaxations can be arbitrarily large. 
In this section, we look at the relative bounding error, namely the ratio between 
the length of the bounding interval and that of the exact interval.
 We shall show that the relative bounding error can be arbitrarily large as well. 
First, for $\gP_1$, we shall prove the following statement.

 \begin{restatable}{thm}{s}\label{thm:relative_error}
 Let $d\in \mathbb{N}$ and let $X\in \mathbb{R}^d$ be  a convex polytope.
 For all $T>0$, there  
 exist 
 a ReLU network $f:\mathbb{R}^d\rightarrow \mathbb{R}$, such that 
 \[
 u(f, \mathcal{P}_1,X)  - \ell(f, \mathcal{P}_1,X) \ge T\bigl(\max(f(X)) - \min(f(X))\bigr)
 \]
\end{restatable}

\begin{proof}
  Without loss of generality, we prove the case when $T>1$; otherwise, we can simply take the threshold as $\max(1, T)$ in the proof. Further, let $X=[-1,1]$; otherwise, we can first project $X$ to one of its non-empty dimensions and scale the projected set by a single affine layer, without changing the output range of any subsequent network and the bounds computed by $\gP_1$. 
  
  Let the ReLU network $f_1 =  \rho\circ W_1$, 
where 
$W_1$ is the affine transformation
$  W_1(\vx) := \begin{pmatrix}
    1\\
   -1 
\end{pmatrix} \vx+ \begin{pmatrix}
    0\\1
\end{pmatrix}, \text{ for } \vx \in \sR^2.$
The function $ f_1$ maps $X$ into the set $\{\vx\in \sR^2: \vx_1\in [0,1], \vx_2 = -\vx_1+1 \}\cup \{\vx\in \sR^2: \vx_1=0,1\leq \vx_2\leq 2\}$,
whose convex hull is $\{\vx_2\leq -2\vx_1+2, \vx_1\geq 0, \vx_2 \geq -\vx_1+1\}$. 
Now consider the function $h(\vx) = \vx_1(\vx_2+\vx_1-1)$, which is constantly zero on the set $f_1(X)$.
We have that 
\[
\min(h(\conv(f_1(X)))) = 0, \hspace{0.5cm} \delta:= \max(h(\conv(f_1(X)))) >0.
\]
Scaling $h$ by $2T/\delta$ gives 
\[
\min(\frac{2T}{\delta}h(\conv(f_1(X)))) = 0, \hspace{0.5cm} \max(\frac{2T}{\delta} h(\conv(f_1(X)))) = 2T.
\]

By the universal approximation \citep{aroraUNDERSTANDINGDEEPNEURAL2018b}, 
there exist a ReLU network $f_2$ satisfying 
\[
\sup_{\conv(f_1(X))} |f_2-\frac{2T}{\delta}h| \leq \frac{1}{2}
\]
Therefore, 
\[
\min(f_2\circ f_1)(X) \geq \min(\frac{2T}{\delta}h\circ f_1)(X) -\frac{1}{2} = -\frac{1}{2},
\]
\[
  \max(f_2\circ f_1)(X) \leq \max(\frac{2T}{\delta}h\circ f_1)(X) +\frac{1}{2} = \frac{1}{2},
\]
and 
\[
\min f_2(\conv(f_1(X))) \leq \min(\frac{2T}{\delta}h(\conv(f_1(X)))) + \frac{1}{2} = \frac{1}{2},
\]
\[
  \max f_2(\conv(f_1(X))) \geq \max(\frac{2T}{\delta}h(\conv(f_1(X)))) - \frac{1}{2} = 2T - \frac{1}{2}.
\]
Taking $f=f_2\circ f_1$,
by \cref{lem:bound_cannot_improve} 
we know that 
\[u(f, \mathcal{P}_1,X)  - \ell(f, \mathcal{P}_1,X) \geq \max f_2(\conv(f_1(X))) - \min f_2(\conv(f_1(X))) \geq 2T-1
\]
and 
\[
  \max(f_2\circ f_1)(X)- \min(f_2\circ f_1)(X) \leq 1.
\]
Hence, since $T>1$ and $\max(f(X))-\min(f(X))\leq 1$,
\[
  u(f, \mathcal{P}_1,X)  - \ell(f, \mathcal{P}_1,X)
  \geq 2T-1 > T \geq T\bigl(\max(f(X))-\min(f(X))\bigr).
\]
\end{proof}

We proceed to show that the relative bounding error established above for $\gP_1$ extends to all cross-layer relaxations. 
Just as in \cref{sec:cross_layer}, we do not consider specific $\gP_r$ for some fixed $r\in \sN$,  but rather 
directly look at
the fully general case $\gP_{\max(1, \floor{\alpha L})}$ where the cross-layer is allowed to depend on the network depth $L$. 
Formally, we shall show

 \begin{restatable}{thm}{s}\label{thm:relative_error_cross}
 Let $d\in \mathbb{N}$ and let $X\in \mathbb{R}^d$ be  a convex polytope.
 For all $T>0$, there  
 exist 
 a ReLU network $f:\mathbb{R}^d\rightarrow \mathbb{R}$ of depth $L$, such that 
 \[
 u(f, \gP_{\max(1, \floor{\alpha L})},X)  - \ell(f, \gP_{\max(1, \floor{\alpha L})},X) \ge T\bigl(\max(f(X)) - \min(f(X))\bigr).
 \]
\end{restatable}

\begin{proof}
  Without loss of generality, we prove the case when $T>1$; otherwise, we can simply take the threshold as $\max(1, T)$ in the proof.

    We reuse the construction in the proof of \cref{thm:relative_error} and augment it by \cref{lem:pumping_lemma}.
    Specifically, in the proof of  \cref{thm:relative_error}, we constructed a ReLU network $f = f_2\circ f_1$
    satisfying 
\[
  \max(f(X))- \min(f(X)) \leq 1.
\]
and 
\[ \max f_2(\conv(f_1(X))) - \min f_2(\conv(f_1(X))) \geq 2T-1
\]
Now by  \cref{lem:pumping_lemma}, for some $L\in \sN$, there exist 
an $L$-layer network $\hat{f}$ such that $\hat{f} = f$ everywhere on $X$ and
\[u(\hat{f}, \gP_{\max(1, \floor{\alpha L})},X) \geq \max f_2(\conv(f_1(X))), \]
\[ \ell(\hat{f}, \gP_{\max(1, \floor{\alpha L})},X) \leq \min f_2(\conv(f_1(X))).
\]
Therefore, 
\begin{align*}
  &\quad u(\hat{f}, \gP_{\max(1, \floor{\alpha L})},X)  - \ell(\hat{f}, \gP_{\max(1, \floor{\alpha L})},X) \\
  &\geq \max f_2(\conv(f_1(X))) - \min f_2(\conv(f_1(X))) \\
  &\geq 2T-1.
\end{align*}
Hence, since $T>1$ and $\max(\hat{f}(X))-\min(\hat{f}(X))=\max(f(X))-\min(f(X))\leq 1$,
\begin{align*}
 &  \quad u(\hat{f}, \gP_{\max(1, \floor{\alpha L})},X)  - \ell(\hat{f}, \gP_{\max(1, \floor{\alpha L})},X) \\
 &\geq 2T-1 \\
 &> T \\
 &\geq T\bigl(\max(\hat{f}(X))-\min(\hat{f}(X))\bigr).
\end{align*}

\end{proof}

 \section{Extension to non-polynomial activation functions} \label{app:non-relu}

In this section, we extend the incompleteness results for ReLU networks to common non-ReLU activations, including logistic sigmoid and $\tanh$.
The layerwise result (\cref{thm:infinite_relaxation_error}) extends to all continuous non-polynomial activation functions.
The cross-layer result (\cref{thm:cross_layer}) extends to admissible activation functions as defined in \cref*{def:injective_activation}.

The proof of the incompleteness of $\mathcal{P}_1$
relies on the following insight: by universal approximation with non-polynomial activation functions  \citep[Theorem~1]{Leshno1993}, we can always construct a network that approximates the construction for ReLU networks with arbitrary precision.

We first introduce some notations. Let $H$ and $H^\prime$ be two sets in $\sR^d$.
Then, we define the Hausdorff distance (induced by the $\ell_2$ norm) between $H$ and $H^\prime$ as 
\[
D(H,H^\prime) := \max\{\sup_{\vx \in H} \inf_{\vy \in H^\prime} \|\vx - \vy\|_2, \sup_{\vy \in H^\prime} \inf_{\vx \in H} \|\vx - \vy\|_2\}.
\]
We will use two properties of the Hausdorff distance. First, $D(H, H^\prime)$ satisfies the triangle inequality (we omit the proof since it is a standard result), i.e., for any three sets $H_1,H_2,H_3$ in $\sR^d$, we have
\[
D(H_1,H_3) \leq D(H_1,H_2) + D(H_2,H_3).
\]
Second, $H \rightarrow \conv(H)$ is $1$-Lipschitz with respect to the Hausdorff distance, stated as follows.
\begin{restatable}{lem}{convexhulllipschitz} \label{lem:convex_hull_lipschitz}
  For any two nonempty compact sets $H_1,H_2$ in $\sR^d$, we have
  \[
  D(\conv(H_1), \conv(H_2)) \leq D(H_1,H_2).
  \]
\end{restatable}
\begin{proof}
  We prove that $\sup_{\vx \in \conv(H_1)} \inf_{\vy \in \conv(H_2)} \|\vx - \vy\|_2 \leq D(H_1,H_2)$; the other side can be proven by symmetry.

  Fix an arbitrary $\vx \in \conv(H_1)$. By definition of convex hull, there exist $k\in \sN^+$, $\lambda_i\geq 0$ for $i\in [k]$ with $\sum_{i=1}^k \lambda_i = 1$, and $\vx_i \in H_1$ for $i\in [k]$ such that $\vx = \sum_{i=1}^k \lambda_i \vx_i$.
  By definition of Hausdorff distance, for each $i\in [k]$, there exists $\vy_i \in H_2$ such that $\|\vx_i - \vy_i\|_2 \leq D(H_1,H_2)$.
  Let $\vy = \sum_{i=1}^k \lambda_i \vy_i$. Then, by Jensen's inequality and note that $\|\cdot\|_2$ is convex, we have
  \begin{align*}
    \|\vx - \vy\|_2 &= \|\sum_{i=1}^k \lambda_i (\vx_i - \vy_i)\|_2\\
    &\leq \sum_{i=1}^k \lambda_i \|\vx_i - \vy_i\|_2\\
    &\leq D(H_1,H_2).
  \end{align*}
  This implies that $\inf_{\vy \in \conv(H_2)} \|\vx - \vy\|_2 \leq D(H_1,H_2)$. Since $\vx$ is arbitrary, we finalize the proof.
\end{proof}

Now we are ready to present the extended version of \cref{thm:infinite_relaxation_error} for non-polynomial activation functions. We will show that for any non-polynomial activation $\sigma$, there exists a sub-network $f_1^\sigma$ and an input polytope $X$ such that $\conv(f_1^\sigma(X))$ is a strict superset of $f_1^\sigma(X)$. Further, for the function $f_2(x; c) := (x-c)^2$ where the point $c \in \conv(f_1^\sigma(X)) \setminus f_1^\sigma(X)$, there exists a network $f_2^\sigma$ approximating $f_2$ on $\conv(f_1^\sigma(X))$ with arbitrary precision. Combining these two results, we can construct a network $f^\sigma = f_2^\sigma \circ f_1^\sigma$ such that the bounding error by any layerwise relaxation is arbitrarily large.

\begin{restatable}{prop}{} \label{prop:non-relu-loose-convex-hull}
  Let $d\in \mathbb{N}$, $\sigma:\sR \rightarrow \sR$ be a non-polynomial activation function and $X\in \mathbb{R}^d$ be  a convex polytope. Then, there exists a network $f^\sigma$, such that the $\conv(f^\sigma(X))$ is 
 a strict superset of $f^\sigma(X)$.
\end{restatable}

\begin{proof}
  Let $f_1$ be some function where $\conv(f_1(X)) \setminus f_1(X)$ is non-empty, e.g., the function constructed in the proof of \cref{thm:infinite_relaxation_error}. By universal approximation, there exists a network $f_1^\sigma$ such that
  \[
    \sup_{X} \|f_1^\sigma - f_1\|_2 \leq \epsilon,
  \]
  for some $\epsilon >0$ to be specified later. Let $H := f_1(X)$ and $H^\prime := f_1^\sigma(X)$. This means
  \[
    D(H,H^\prime) \leq \epsilon.
  \]
  Let $\Delta := D(\conv(H), H)$. Since $\conv(H) \setminus H$ is non-empty, we have $\Delta >0$. By triangle inequality and \cref{lem:convex_hull_lipschitz}, we have
  \begin{align*}
    D(\conv(H), H) &\le D(\conv(H), \conv(H^\prime)) + D(\conv(H^\prime), H^\prime) + D(H^\prime, H) \\
    &\leq 2D(H,H^\prime) + D(\conv(H^\prime), H^\prime) \\
    &\leq 2\epsilon + D(\conv(H^\prime), H^\prime).
  \end{align*}  
  Thus, taking $\epsilon = \Delta / 4$, we have
  \begin{align*}
    D(\conv(H^\prime), H^\prime) &\geq D(\conv(H), H) - 2\epsilon \\
    &= \frac{\Delta}{2} >0.
  \end{align*}
  This implies that $\conv(H^\prime) \setminus H^\prime$ is non-empty, finalizing the proof.
\end{proof}

\begin{restatable}{thm}{s}\label{thm:non-relu-layerwise}
 Let $d\in \mathbb{N}$, let $\sigma:\sR \rightarrow \sR$ be a non-polynomial activation function, and let $X\subseteq \mathbb{R}^d$ be a convex polytope. For every constant $T>0$, there exists a network $f^\sigma:\mathbb{R}^d\rightarrow \mathbb{R}$ such that $\ell(f^\sigma, \mathcal{P}_1,X) \leq \min f^\sigma(X) -T$, and another network $g^\sigma:\mathbb{R}^d\rightarrow \mathbb{R}$ such that $u(g^\sigma, \mathcal{P}_1,X) \geq \max g^\sigma(X) +T$.
\end{restatable}

\begin{proof}

  By \cref{prop:non-relu-loose-convex-hull}, there exists a network $f_1^\sigma$ such that $\conv(f_1^\sigma(X))$ is a strict superset of $f_1^\sigma(X)$. Let $\vc \in \conv(f_1^\sigma(X)) \setminus f_1^\sigma(X)$ such that $\delta := \min_{\vh \in f_1^\sigma(X)} \|\vh - \vc\|_2 > 0$. Let $f_2(\vh) := \|\vh - \vc\|_2$. Thus, we have
  \[
  \min_{\vh \in f_1^\sigma(X)} f_2(\vh) = \delta, \quad \min_{\vh \in \conv(f_1^\sigma(X))} f_2(\vh) = 0.
  \]
  By universal approximation, there exists a network $f_2^\sigma$ such that
  \[
    \sup_{\conv(f_1^\sigma(X))} |f_2^\sigma - \frac{2T}{\delta}f_2| \leq \epsilon,
  \]
  for some $\epsilon >0$ to be specified later. Let $f^\sigma := f_2^\sigma \circ f_1^\sigma$. Then, we have
  \begin{align*}
    \min f^\sigma(X) &\geq \min_{\vh \in f_1^\sigma(X)} \frac{2T}{\delta}f_2(\vh) - \epsilon = 2T - \epsilon, \\
    \ell(f^\sigma, \mathcal{P}_1,X) &\leq \min_{\vh \in \conv(f_1^\sigma(X))} \frac{2T}{\delta}f_2(\vh) + \epsilon = \epsilon.
  \end{align*}
  Thus, we have
  \[
    \ell(f^\sigma, \mathcal{P}_1,X) - \min f^\sigma(X) \leq -2T + 2\epsilon.
  \]
  Letting $\epsilon = T/2$ gives
  \[
    \ell(f^\sigma, \mathcal{P}_1,X) \leq \min f^\sigma(X)-T,
  \]
  which proves the lower-bound statement.

  For the upper bound, let $h^\sigma$ be the network constructed above for the lower-bound case and define $g^\sigma:=-h^\sigma$ by negating the final affine layer. The layerwise relaxation of $g^\sigma$ is obtained from the layerwise relaxation of $h^\sigma$ by the affine change of the output variable $y_g=-y_h$. Hence, by the definitions of $\ell$ and $u$,
  \[
    u(g^\sigma,\mathcal{P}_1,X)=-\ell(h^\sigma,\mathcal{P}_1,X),
    \qquad
    \max g^\sigma(X)=-\min h^\sigma(X).
  \]
  Therefore,
  \[
    u(g^\sigma,\mathcal{P}_1,X)
    =-\ell(h^\sigma,\mathcal{P}_1,X)
    \geq -\min h^\sigma(X)+T
    =\max g^\sigma(X)+T.
  \]
  This completes the proof.
\end{proof}

We proceed to establish the incompleteness of $\gP_{\max(1, \lfloor \alpha L\rfloor)}$.
For the sake of technical convenience, we restrict attention to activation functions satisfying the additional technical conditions below.

\begin{defi}\label{def:injective_activation}
  Let $\sigma:\sR \rightarrow \sR$ be a continuous non-polynomial function.  
  We call $\sigma$ an \emph{admissible} activation function if it satisfies the following conditions:
  \begin{itemize}
    \item $\sigma$ is injective;
    \item   there exists $a, b\in \sR, a<b$ satisfying 
  \[
    \sigma(\frac{a+b}{2}) = \frac{\sigma(a)+\sigma(b)}{2}
  \]
  \end{itemize}
\end{defi}

\begin{remark}
  The technical conditions in \cref{def:injective_activation} are satisfied by many common activation functions.
  We list some examples here. 
  \begin{itemize}
    \item \textbf{Logistic sigmoid.} For $\sigma(t)=1/(1+e^{-t})$, we have
    $\sigma'(t)=e^{-t}/(1+e^{-t})^2>0$, so $\sigma$ is strictly increasing and thus injective.
    Moreover, $\sigma(-t)=1-\sigma(t)$, giving
    \[
      \frac{\sigma(-c)+\sigma(c)}{2}=\frac12=\sigma(0)
      =\sigma\!\left(\frac{a+b}{2}\right).
    \]
    \item \textbf{Hyperbolic tangent ($\tanh$).} For $\sigma(t)=\tanh t$, we have
    $\sigma'(t)=1/\cosh^2 t>0$, so $\sigma$ is injective.
    Since $\sigma$ is odd,
    \[
      \frac{\sigma(-c)+\sigma(c)}{2}=0=\sigma(0)
      =\sigma\!\left(\frac{a+b}{2}\right).
    \]
    \item \textbf{Arctangent.} For $\sigma(t)=\arctan t$, we have
    $\sigma'(t)=1/(1+t^2)>0$, so $\sigma$ is injective.
    Oddness again yields
    \[
      \frac{\sigma(-c)+\sigma(c)}{2}=0=\sigma(0)
      =\sigma\!\left(\frac{a+b}{2}\right).
    \]
    \item \textbf{Softsign.} For $\sigma(t)=t/(1+|t|)$, we have
    $\sigma'(t)=1/(1+|t|)^2>0$ for $t\neq0$, while
    $\sigma'(0)=\lim_{h\to0}1/(1+|h|)=1$.
    Thus $\sigma$ is strictly increasing on $\sR$ and hence injective.
    It is also odd, so
    \[
      \frac{\sigma(-c)+\sigma(c)}{2}=0=\sigma(0)
      =\sigma\!\left(\frac{a+b}{2}\right).
    \]
  \end{itemize}
\end{remark}

We follow the pumping idea of \cref{thm:cross_layer}, but 
instead of 
identity layers, here 
we construct a bridge that preserves a midpoint and its relation to two true
trajectories at every layer, ensuring compatibility with overlapping
cross-layer constraints.
Injectivity keeps this spurious trajectory outside the true reachable
set; a final subnetwork then turn the trajectory into an effective bounding error.

\begin{lem}\label{lem:cross-layer-bridge}

Let $\sigma:\sR\rightarrow\sR$ be an admissible activation function.  For every $d,M\in\sN^+$ and every compact noncovex set $S\subseteq\sR^d$, 
there exist two points $\vx_1, \vx_2 \in S$ and 
a $\sigma$-network
  $f^\sigma:\sR^d\rightarrow\sR^d, f^\sigma = \psi_L\circ \cdots \circ \psi_1$, of depth $L\geq M$ satisfying 
    \begin{align*}
    & f_j(\frac{\vx_1+\vx_2}{2}) = \frac{f_j(\vx_1)+f_j(\vx_2)}{2} \notin f_j(S)
    \qquad j=1,\ldots,L.
  \end{align*}
where $f_j = \psi_j\circ\cdots\circ\psi_1, j = 1,\ldots, L$.
\end{lem}
\begin{proof}
  Since $S$ is compact and nonconvex, there exist $\vx_1,\vx_2\in S$ such that $\vz:=(\vx_1+\vx_2)/2\notin S$.
Because $\sigma$ is admissible, we can 
  choose $a<b$ as in \cref{def:injective_activation}, and set $m:=(a+b)/2$.
  Let $\vone$ denote the all-ones vector in $\sR^d$.
  Since $\vx_1\neq\vx_2$, there exists an invertible linear map $\mW:\sR^d\rightarrow\sR^d$ satisfying
  \[
    \mW(\vx_2-\vx_1)=(b-a)\vone.
  \]
  Define the invertible affine map $A(\vx):=\mW(\vx-\vx_1)+a\vone$. Then,
  \[
    A(\vx_1)=a\vone,\qquad A(\vx_2)=b\vone,\qquad A(\vz)=m\vone.
  \]
  By injectivity of $\sigma$, the scalar
  \[
    \gamma:=\frac{b-a}{\sigma(b)-\sigma(a)}
  \]
  is well-defined and nonzero. Thus the affine map
  \[
    R(\vv):=m\vone+\gamma\bigl(\vv-\sigma(m)\vone\bigr)
  \]
  is invertible. Since $\sigma(m)=(\sigma(a)+\sigma(b))/2$, we have
  \[
    R(\sigma(a)\vone)=a\vone,\qquad
    R(\sigma(m)\vone)=m\vone,\qquad
    R(\sigma(b)\vone)=b\vone.
  \]

  Choose $k\in\sN^+$ such that $L:=2k+1\geq M$, and define
  \[
    \psi_1:=A,\qquad \psi_{2i}:=\sigma,\qquad
    \psi_{2i+1}:=R,\qquad i\in[k],
  \]
  where $\sigma$ acts coordinatewise.
  The resulting network $f^\sigma=(R\circ\sigma)^k\circ A$ has depth $L$ and alternates affine and activation layers.
  The preceding identities show by induction that, at every odd-indexed prefix, the images of $\vx_1,\vx_2,\vz$ are $a\vone,b\vone,m\vone$, respectively, and at every even-indexed prefix they are $\sigma(a)\vone,\sigma(b)\vone,\sigma(m)\vone$.
  Therefore,
  \[
    f_j(\vz)=\frac{f_j(\vx_1)+f_j(\vx_2)}{2},\qquad j\in[L].
  \]
  Moreover, every layer is injective, so every prefix $f_j$ is injective.
  If $f_j(\vz)\in f_j(S)$, injectivity would imply $\vz\in S$, a contradiction.
  Hence $f_j(\vz)\notin f_j(S)$ for every $j\in[L]$, which completes the proof.
\end{proof}

\begin{restatable}{thm}{s}\label{thm:thm:new_non-relu-crosslayer}
 Let $d\in \mathbb{N}$, let $\sigma:\sR \rightarrow \sR$ be an admissible activation function.
Let $X\subseteq \mathbb{R}^d$ be a convex polytope. For every $\alpha \in (0,1)$ and
every constant $T>0$, there exists a network $f\in \mathcal{N}^{\sigma}$ of depth $L$, $ f:\mathbb{R}^d\rightarrow \mathbb{R}$, such that $\ell(f, \mathcal{P}_{\max(1, \lfloor \alpha L\rfloor)},X) \leq \min f(X) -T$, and another network $g\in \mathcal{N}^{\sigma}$ of depth $L$ such that $u(g, \mathcal{P}_{\max(1, \lfloor \alpha L\rfloor)},X) \geq \max g(X) +T$.
\end{restatable}

\begin{proof}
  It suffices to prove the lower-bound statement. The upper-bound statement follows by negating the final affine layer.

  By \cref{prop:non-relu-loose-convex-hull}, there exists a $\sigma$-network $H:\sR^d\rightarrow\sR^{d'}$ whose image $S:=H(X)$ is nonconvex.
  Since $X$ is compact and $\sigma$ is continuous, $S$ is compact.
  Let $L_0$ denote the depth of $H$. Choose an integer $M\geq3$ such that
  \[
    \alpha(L_0+M+1)\leq M-1,
  \]
  which is possible because $\alpha<1$.
  Applying \cref{lem:cross-layer-bridge} to $S$, we obtain points $\vs_1,\vs_2\in S$ and a network $B=\psi_N\circ\cdots\circ\psi_1$ of depth $N\geq M$.
  Write $\vz:=(\vs_1+\vs_2)/2$ and $B_j:=\psi_j\circ\cdots\circ\psi_1$.
  The lemma gives
  \begin{equation}\label{eq:admissible_bridge_prefix}
    B_j(\vz)=\frac{B_j(\vs_1)+B_j(\vs_2)}{2}\notin B_j(S),
    \qquad j\in[N].
  \end{equation}

  We use the prefix $B_{N-1}$, omitting the final affine layer of $B$, so that this prefix ends with an activation layer.
  Set
  \[
    K:=B_{N-1}(S),\qquad \vp:=B_{N-1}(\vz).
  \]
  By \eqref{eq:admissible_bridge_prefix}, $\vp\notin K$. Since $K$ is compact,
  \[
    \delta:=\min_{\vh\in K}\|\vh-\vp\|_2>0.
  \]
  By the universal approximation theorem for continuous non-polynomial activations \citep[Theorem~1]{Leshno1993}, there exists a one-hidden-layer network
  \[
    D(\vh)=\vu^\top\sigma(\mA\vh+\vb)+c
  \]
  satisfying
  \[
    \sup_{\vh\in\conv(K\cup\{\vp\})}
    \left|D(\vh)-\frac{2T}{\delta^2}\|\vh-\vp\|_2^2\right|
    \leq\frac{T}{2}.
  \]
  Consequently,
  \begin{equation}\label{eq:admissible_detector_bounds}
    D(\vp)\leq\frac{T}{2},\qquad
    \min_{\vh\in K}D(\vh)\geq\frac{3T}{2}.
  \end{equation}

  Define $f:=D\circ B_{N-1}\circ H$.
  To retain alternating affine and activation layers, merge the first affine layer $\psi_1$ of $B$ with the final affine layer of $H$.
  The resulting front subnetwork $\widetilde H:=\psi_1\circ H$ still has depth $L_0$.
  It is followed by the $N-2$ layers $\psi_2,\ldots,\psi_{N-1}$, which start and end with activation layers, and then by the three layers of $D$.
  Thus $f$ has depth
  \[
    L:=L_0+N+1.
  \]
  Since $N\geq M$ and $\alpha<1$, our choice of $M$ implies
  \[
    \alpha L=\alpha(L_0+N+1)\leq N-1.
  \]
  Setting $r:=\max(1,\lfloor\alpha L\rfloor)$, we therefore have
  \[
    r\leq N-1,\qquad r-1\leq N-2.
  \]

  We next construct a feasible assignment for the $\gP_r$ relaxation of $f$ on $X$.
  Choose $\vx_1,\vx_2\in X$ such that $H(\vx_i)=\vs_i$ for $i\in\{1,2\}$.
  Denote the variable after layer $i$ of $f$ by $\vv^{(i)}$, with $\vv^{(0)}$ the input variable.
  Up to and including layer $L_0$, assign each variable the midpoint of its values on the two exact trajectories starting from $\vx_1$ and $\vx_2$.
  The assigned input belongs to $X$ by convexity, and the assigned value at layer $L_0$ is
  \[
    \frac{\widetilde H(\vx_1)+\widetilde H(\vx_2)}{2}
    =B_1(\vz).
  \]
  After this layer, assign all variables by exact evaluation of the remaining layers from $B_1(\vz)$.
  Thus the bridge variables take the values $B_j(\vz)$ for $j=2,\ldots,N-1$, and the output is $D(\vp)$.

  Consider first a window starting at a variable $\vv^{(i)}$ with $i<L_0$.
  Its last layer index is at most
  \[
    i+r\leq L_0-1+r\leq L_0+N-2,
  \]
  so the window ends before the detector begins.
  By \eqref{eq:admissible_bridge_prefix}, every assigned coordinate in this window is the midpoint of the corresponding coordinates on the two exact trajectories.
  Both exact trajectories remain feasible throughout the relaxation by soundness. Hence their window tuples belong to the exact graph used to define this window, and their midpoint belongs to its convex hull.
  All windows starting before layer $L_0$ are therefore satisfied by this single assignment.

  For the remaining windows, proceed in the order in which $\gP_r$ instantiates them.
  Suppose that all earlier window constraints are satisfied.
  Then the assigned input of the current window belongs to the projected feasible set used to instantiate that window.
  Since the window starts at or after layer $L_0$, its assigned tuple is the exact continuation of that input through the corresponding layers.
  It therefore belongs to the exact graph, and hence to the convex hull defining the window constraint.
  This proves by induction that the complete assignment satisfies every window constraint.
  In particular, $D(\vp)$ belongs to the relaxed output set, so \eqref{eq:admissible_detector_bounds} gives
  \[
    \ell(f,\gP_r,X)\leq D(\vp)\leq\frac{T}{2}.
  \]
  On the other hand, every true input reaches a point in $K$ before the detector. Thus,
  \[
    \min f(X)=\min_{\vh\in K}D(\vh)\geq\frac{3T}{2}.
  \]
  Combining these inequalities yields
  \[
    \ell(f,\gP_r,X)\leq\min f(X)-T.
  \]

  Finally, define $g:=-f$ by negating the final affine layer.
  This preserves the depth $L$ and maps each relaxed output value $y$ to $-y$. Therefore,
  \[
    u(g,\gP_r,X)=-\ell(f,\gP_r,X)
    \geq-\min f(X)+T
    =\max g(X)+T.
  \]
  Since $r=\max(1,\lfloor\alpha L\rfloor)$, this completes the proof.
\end{proof}

 \section{LLM Usage}

 LLMs (GPT-5) were used to polish the writing of the paper, and were not used for any other purpose.

\end{document}